\pdfoutput=1
\documentclass[11pt]{article}
\usepackage{caption}
\usepackage{subcaption}
\usepackage{amsmath,amssymb,amsfonts,mathrsfs,mathtools,bm,bbm, dsfont}
\usepackage{graphicx}
\usepackage{pifont,booktabs} 
\usepackage{float}      
\usepackage{adjustbox} 

\usepackage[margin=1in]{geometry}
\usepackage{amsmath,amssymb,amsthm,amsfonts,mathtools,bm,bbm,mathrsfs,dsfont}
\usepackage{mathtools}
\usepackage{enumitem}
\usepackage{hyperref}
\usepackage{mathrsfs}
\usepackage{graphicx}
\usepackage{subcaption}
\usepackage{microtype}
\usepackage[dvipsnames]{xcolor}
\usepackage{graphicx}
\usepackage{subcaption}
\usepackage{booktabs}
\usepackage{float}
\usepackage{adjustbox}
\usepackage{multicol}
\usepackage{accents,makecell}
\usepackage{comment}
\usepackage{algorithm}
\usepackage{algorithmic}
\usepackage[textsize=tiny]{todonotes}
\usepackage{hyperref}
\usepackage[capitalize,noabbrev]{cleveref}
\usepackage{cite}
\setlist[itemize]{leftmargin=*,nolistsep,noitemsep}
\setlist[enumerate]{leftmargin=*,nolistsep,noitemsep}
\usepackage{tikz}
\usetikzlibrary{arrows.meta, positioning, fit, calc}

\DeclareMathOperator*{\argmax}{arg\,max}

\newcommand{\Acal}{\mathcal{A}}
\newcommand{\Bcal}{\mathcal{B}}

\newcommand{\Fcal}{\mathcal{F}}

\newcommand{\Ocal}{\mathcal{O}}

\newcommand{\Rset}{\mathbb{R}}
\newcommand{\E}{\mathbb{E}}
\newcommand{\ve}{\varepsilon}

\newcommand{\vnorm}[1]{\left\|#1\right\|}
\newcommand{\probb}[1]{\mathbb{P}\!\left\{#1\right\}}

\theoremstyle{plain}
\newtheorem{theorem}{Theorem}[section]

\newtheorem{lemma}[theorem]{Lemma}
\newtheorem{corollary}[theorem]{Corollary}
\theoremstyle{definition}

\newtheorem{assumption}[theorem]{Assumption}
\newtheorem{property}[theorem]{Property}
\theoremstyle{remark}

\title{\LARGE \bf
A Unified Optimism-Agnostic Framework for Linear Bandits 
over Spherical Action Sets
}

\author{Arda G\"u\c{c}l\"u\qquad Subhonmesh Bose \qquad John R. Birge\thanks{A. Guclu and S. Bose are with the Department of Electrical and Computer Engineering and the Coordinated Science Laboratory from the University of Illinois Urbana-Champaign, Urbana, IL  61801. J. R. Birge is with the University of Chicago Booth School of Business, Chicago, IL 60637. Emails: \href{mailto@illinois.edu}{aguclu2@illinois.edu},
\href{mailto@illinois.edu}{boses@illinois.edu},
\href{mailto.Birge@ChicagoBooth.edu}{John.Birge@ChicagoBooth.edu}. This work was partially supported by the U.S. National Science Foundation under Grant No. ECCS-EPCN 2349418.}}

\begin{document}

\maketitle

\begin{abstract}
Linear bandits model sequential decision-making problems with noisy rewards that are linear in the decision variable, where an agent must simultaneously learn about an unknown  parameter that governs the mean rewards, while maximizing (expected) rewards over time. Two prominent algorithmic families--upper confidence bound (UCB) and Thompson sampling (TS)--achieve a balance of exploration (to estimate said parameter) and exploitation (utilization of knowledge about it) across time. The quality of estimation of that parameter depends on the eigenvalues of a design matrix. In this paper, we begin by showing that if the inference quality obtained from exploration, encoded in the minimum eigenvalue of the design matrix, grows $\gtrsim \sqrt{t}$ with time $t$, while actions remain sufficiently concentrated for exploitation, then an algorithm produces optimal high-probability $\mathcal{O}(\sqrt{T}\log T)$-regret rate over a time-horizon $T$ for spherical action sets. This analysis is algorithm-agnostic and follows an alternative route to the classical optimism-based elliptical-potential argument for regret analysis. Then, we illustrate that variants of UCB and TS satisfy the inference and concentration properties and in turn, enjoy optimal regret rate. In effect, our results provide a modular framework that can be used to analyze linear bandit algorithms and explicitly connect quality of parameter estimation to optimal regret accumulation.
\end{abstract}

\section{Introduction}
Sequential decision-making requires balancing immediate reward with information acquisition. Actions that are optimal under current knowledge may be uninformative, while informative actions may incur immediate loss. Regret from suboptimal actions therefore encapsulates two costs---insufficient precision in estimating the unknown environment and the cost of acquiring information to improve that precision.
We study this tradeoff in stochastic linear bandits with spherical action sets. The spherical geometry yields a decomposition of one-step regret into a precision cost, due to error in the estimated optimal action, and an information-acquisition cost due to deviation from that action. This reduces optimal regret to two requirements: sufficient accumulation of information accompanied by controlled exploration. We encode these two properties mathematically and illustrate how they yield order-optimal regret accumulation.

We verify these properties for linear Thompson Sampling (TS) algorithm from \cite{linear_thompson_sampling_revisited} and the Upper Confidence Bound (UCB) algorithm from \cite{bandit_algorithms_book, improved_algorithms_for_stochastic_bandits} with a modified confidence set. These two algorithms have completely different exploration mechanisms. The former relies on randomization, and the latter is a deterministic algorithm that is based on the principle of \emph{optimism in the face of uncertainty} (OFUL). We demonstrate how these two algorithms are in fact tied together in the rate at which they both accumulate information and incur exploration costs. Their optimal regret guarantees then follow from a common analytical pipeline.

An agent in a linear bandit environment seeks to maximize the cumulative reward $\sum_{t=1}^T Y_t := \sum_{t=1}^T (\theta^{\star \top} a_t+\ve_t )$ over a time-horizon $T$, upon choosing actions from a compact set $a_t \in \Acal \subset \Rset^n$ at time $t$, without knowing the governing parameter $\theta^\star \in \Theta \subset \Rset^n$, mapping history of rewards and actions up until a point to an action at the next time step. The zero-mean noise $\ve_t$ is assumed independent from that history and the action $a_t$. A reward-maximizing action $a^\star(\theta^\star) \in \argmax_{a \in \Acal} \ \theta^{\star\top} a$ is assumed to exist, with optimal expected per-stage reward, $ \left(\theta^{\star}\right)^{\top} a^{\star}\left(\theta^{\star}\right)$. {We examine spherical action sets defined by
$
\mathcal A
:=
\left\{
a\in\mathbb R^n:\|a\|_2=1
\right\},
$
for which the optimal-action map is
$
a^\star(\theta)=\theta/\|\theta\|
$
whenever $\theta\neq0$. For completeness, when $\theta=0$, we let $a^\star(0)$ be any fixed element of $\mathcal A$.} For any action selection policy, the cumulative expected regret up to time $T$ is,
\begin{align}
 \mathscr{R}_{\theta^\star}(T)
    := \sum_{t=1}^T  \underbrace{ \left( \theta^{\star\top} a^{\star}(\theta^\star)- \theta^{\star\top} a_t \right)}_{:= r_{\theta^\star}(a_t)} .
    \label{eq:regret.def}
\end{align}
For linear bandits, under fairly general assumptions, regret is $\Omega(\sqrt{T})$\footnote{We use the standard notation $\Ocal$, $\Omega$, $\lesssim$, $\gtrsim$, and $\simeq$. In the introduction, we suppress logarithmic factors.} and popular algorithms such as Thompson sampling (TS) and upper confidence bound (UCB) achieve that rate (within polylog factors); see \cite{improved_algorithms_for_stochastic_bandits, linear_thompson_sampling_revisited, contextual_bandits_with_ucb,thompson_sampling_for_contextual_bandits_with_linear_payoffs,context_bandits_ucb, russo2014learning}.  Given a sequence of observations $Y_1, \ldots, Y_{t}$, one can estimate $\theta^\star$ via
\begin{align}
\begin{aligned}
    \widehat{\theta}_{t + 1}
    & :=\min_{\theta}\left[ \sum_{s=1}^{t}\left(\theta^\top a_s-Y_s\right)^2+\lambda \|\theta\|^2\right] \\
    & = \underbrace{\left({\lambda  I_n + \sum_{s=1}^{t} a_s a_s^\top }\right)^{-1}}_{:=V_t^{-1}} \left(\sum_{s = 1}^t a_s Y_s \right)
    \label{eq:Vt.def}
    \end{aligned}
\end{align}
with a regularization parameter $\lambda >0$ and the identity matrix $I_n$. The \emph{design matrix} $V_t \in \Rset^{n \times n}$, and its minimum eigenvalue will play a crucial role in the sequel.
A point-estimate-based control will lead to the choice $a_{t+1} = a^{\star}(\widehat{\theta}_{t +1})$. Such a control policy can converge to a wrong parameter estimate, (see \cite{incomplete_learning_example}) and incur linear regret. One needs to explore away from myopically exploiting the current parameter estimate to avoid such incomplete learning. 

Thompson Sampling (TS) and Upper Confidence Bound (UCB) are two classical algorithms known to balance exploration and exploitation and to achieve order-optimal regret in stochastic linear bandits. TS maintains a distribution over the unknown parameter at each round, samples a parameter $\theta_t$ from this distribution, and plays the corresponding greedy action $a_t=a^\star(\theta_t)$. The distribution is then updated using Bayes' rule. A computationally simpler variant, called linear Thompson sampling, replaces Bayesian posterior sampling with sampling from a $V_t$-colored distribution centered at the regularized least-squares estimate \cite{linear_thompson_sampling_revisited}. This construction naturally adapts the amount of exploration to the available information: when the eigenvalues of $V_t$ are small, the sampling distribution is broad, encouraging exploration, whereas as the eigenvalues increase, the distribution concentrates around the estimator.

UCB follows a different philosophy based on the principle of optimism in the face of uncertainty. At each round, it constructs a confidence region containing $\theta^\star$ with high probability and selects the greedy action corresponding to the most optimistic parameter (the parameter $\theta$ in this region maximizing $\theta^\top a^\star(\theta)$). In the classical algorithm of \cite{improved_algorithms_for_stochastic_bandits}, this confidence region is an ellipsoid whose geometry is determined by $V_t$. As eigenvalues of $V_t$ increase, this confidence region shrinks around $\widehat \theta_t$, and hence, implicitly balances exploration and exploitation.

The failure of purely greedy policies, together with the success of non-myopic methods such as TS and UCB, shows that accurate decision-making requires sufficient exploration. However, effective learning requires not only enough exploration to minimize the cost of unreliable inference, but also enough exploitation to keep the cost of acquiring information under control. On spherical action sets, these two costs appear directly in the geometry of regret. Indeed, if
$
\theta^\star\neq 0
$, then the optimal action is
$
a^\star(\theta^\star)
=
\theta^\star/\|\theta^\star\|.
$
Consequently, for every \(a\in\mathcal{A}\),
\begin{align}
r_{\theta^\star}(a)
=
\theta^{\star\top}
\left(
a^\star(\theta^\star)-a
\right)=
\|\theta^\star\|
\left(
1-a^\star(\theta^\star)^\top a
\right) =
\frac{\|\theta^\star\|}{2}
\left\|
a^\star(\theta^\star)-a
\right\|^2.
\label{eq:spherical_regret_geometry}
\end{align}
Thus, per-step regret is exactly proportional to the squared distance between the played action and the true optimal action. Using the squared triangle inequality,
\begin{align}
r_{\theta^\star}(a_t)
\lesssim\;&
\underbrace{
\left\|
a^\star(\theta^\star)
-
a^\star(\widehat{\theta}_t)
\right\|^2
}_{\text{cost of insufficient precision}}
\nonumber\\
&+
\underbrace{
\left\|
a^\star(\widehat{\theta}_t)
-
a_t
\right\|^2
}_{\text{cost of acquiring information}}.
\label{eq:regret_two_costs}
\end{align}
The first term measures how an imperfect estimate of the problem translates into an imperfect decision. The second term is the 
loss incurred by moving away from the currently greedy action, as may be necessary to improve future inference. The two costs are distinct but dynamically linked.

For the spherical action set, the optimal-action map $a^\star$ is locally Lipschitz away from the origin. In particular, when $\|\theta^\star\| > 0$ and \(\widehat{\theta}_t\) lies in a sufficiently small neighborhood of \(\theta^\star\), we have
\begin{align}
\left\|
a^\star(\theta^\star)
-
a^\star(\widehat{\theta}_t)
\right\|^2
\lesssim
\left\|
\widehat{\theta}_t-\theta^\star
\right\|^2.
\label{eq:optimal_action_lipschitz}
\end{align}
Equations \eqref{eq:regret_two_costs} and \eqref{eq:optimal_action_lipschitz} are the starting points of our analysis. They show that good
performance requires controlling both the quality of the estimate given as  $\|
\widehat{\theta}_t-\theta^\star
\|$ and
the cost paid to improve it. In particular, suppose that, after an initial time
$t_{\min}$, for all $ t \geq t_{\min}$,
\begin{align}
\label{intro_1}
\|\widehat{\theta}_t - \theta^\star\|^2
\lesssim \frac{1}{\sqrt{t}}, \quad
\sum_{t=t_{\min}}^T
\left\|a_t - a^\star(\widehat{\theta}_t)\right\|^2
\lesssim \sqrt{T}.
\end{align}
Then,  \eqref{eq:regret_two_costs} and \eqref{eq:optimal_action_lipschitz} yield
\begin{align}
\sum_{t=t_{\min}}^T r_{\theta^\star}(a_t)
\lesssim
\sum_{t=t_{\min}}^T
\left(
\| \widehat{\theta}_t - \theta^\star\|^2
+
\|a_t - a^\star(\widehat{\theta}_t)\|^2\right)
\lesssim
\sum_{t=t_{\min}}^T \frac{1}{\sqrt{t}}
+
\sqrt{T}
\lesssim
\sqrt{T}.
\label{final_intro}
\end{align}
For bounded $\|\theta^\star\|$, the regret accumulated before $t_{\min}$ is at most of order $t_{\min}$. Hence, the total regret remains
$\mathcal{O}(\sqrt{T})$, provided that
$t_{\min} \lesssim \sqrt{T}$, which is the optimal regret accumulation rate. Thus, \eqref{intro_1} gives us interpretable conditions that  guarantee optimal cumulative regret rate. \emph{Any} algorithm that satisfies those conditions achieves optimal regret. The remaining question is how these conditions can be attained. To this end, notice that \eqref{eq:optimal_action_lipschitz} shows that the precision cost can be controlled through the statistical error of the regularized least-squares estimator. We summarize this precision through the scalar information variable: 
\begin{align}
\Lambda_t
:=
\frac{1}{\operatorname{tr}(V_t^{-1})}.
\label{eq:information_state}
\end{align}
Let $\lambda_{\min,t}$ denote the minimum eigenvalue of $V_{t}$. Classical high-probability regret analyses (e.g., \cite{improved_algorithms_for_stochastic_bandits}) shows that $\theta^\star$ lies inside a confidence ellipsoid around the regularized least-squares estimate $\widehat{\theta}_t$, with
$\|\theta^\star-\widehat{\theta}_t\|_{V_{t-1}}^2 \lesssim \log T$ with high probability. Since
$
\operatorname{tr}(V_t^{-1})
\geq
1/\lambda_{\min,t},
$
we have
$
\lambda_{\min,t}
\geq
\Lambda_t.
$
Thus, standard confidence-ellipsoid bounds imply, up to logarithmic factors,
\begin{align}
\left\|
\widehat{\theta}_t-\theta^\star
\right\|^2
\lesssim
\frac{1}{\lambda_{\min,t-1}}
\lesssim
\frac{1}{\Lambda_{t-1}}.
\label{eq:estimation_information_bound}
\end{align}
Hence, 
\[
\Lambda_t \geq \sqrt t \implies \|\widehat{\theta}_t - \theta^\star\|^2
\lesssim \frac{1}{\sqrt{t}}.
\]

The problem of linear bandits on spherical action sets can therefore be reduced to two questions. How does an algorithm
drive the information variable $\Lambda_t$ to the critical $\sqrt{t}$ scale? How much cumulative information acquisition cost $\sum_{t=t_{\min}}^T
\|a_t - a^\star(\widehat{\theta}_t)\|^2$  does it pay to generate this information?

We show that variants of Thompson Sampling and UCB answer these questions in the same order-wise manner: they produce information increments $\Lambda_{t + 1} - \Lambda_t$ approximately of order \(\Lambda_t^{-1}\), resulting in $\Lambda_t \gtrsim \sqrt t$, \footnote{From $\Lambda_t - \Lambda_{t-1} \ge \frac{C}{\Lambda_{t-1}}$, multiply both sides by $\Lambda_t + \Lambda_{t-1}$ to obtain $\Lambda_t^2 - \Lambda_{t-1}^2 \ge \frac{C}{\Lambda_{t-1}}(\Lambda_t + \Lambda_{t-1}) \ge C$, hence $\Lambda_t^2 \ge \Lambda_{t-1}^2 + C$, and summing over $t$ gives $\Lambda_t^2 \ge \Lambda_0^2 + Ct$, which implies $\Lambda_t \gtrsim \sqrt{t}$.
} while keeping the cost of deviating from the current greedy action $\|a_t - a^\star(\widehat{\theta}_t)\|^2$ at order \(\Lambda_t^{-1}\), resulting in $\sum_t \|a_t - a^\star(\widehat{\theta}_t)\|^2 \leq \sum_t \Lambda_t^{-1} \lesssim \sqrt T $. Their implementations differ, but the information variable trajectories they generate and the prices they pay for them agree at the scale relevant for regret.

In linear bandits, however, the information increment $\Lambda_t - \Lambda_{t-1}$ is typically random. Consequently, to unify the analysis of both randomized algorithms such as TS and optimism-based deterministic algorithms such as UCB, we examine the expected growth of information quality, $\E\left[ \Lambda_t - \Lambda_{t-1}|
 \mathcal{F}_{t-1}\right]$ where $\mathcal{F}_{t-1}$ encodes the history of the algorithm up until time $t-1$. We show that, for certain TS and UCB variants, the following information growth occurs:
\begin{align}
    \E\left[ \Lambda_t -   \Lambda_{t-1}|
 \mathcal{F}_{t-1}\right] \geq   \frac{C_4^{\pi}}{\Lambda_{t-1}} \label{intro:self_correction}
\end{align}
for some $C_4^{\pi}>0$. Ignoring probabilistic considerations, the above condition mirrors the growth of the scalar dynamics,  $\Lambda_{t}- \Lambda_{t-1} \geq  C_4^{\pi}/\Lambda_{t-1}$.  We refer this condition as \emph{self-correction} as it ensures that $\Lambda_t $ self-regulates its growth rate to $\Omega(\sqrt{t})$. When $\Lambda_t$ falls below this rate, \eqref{intro:self_correction} pushes the system to explore more aggressively so as to return to the target growth rate. In Section \ref{sec:guaranteed_inference}, we show that this property is sufficient to achieve the guaranteed inference condition, $\Lambda_t \gtrsim \sqrt{t}$, with high probability after a burn-in period.

At Sections \ref{sec:ts} and \ref{sec:ucb}, we establish the self-correcting property of $\Lambda_t$ for the Thompson Sampling variant of \cite{linear_thompson_sampling_revisited} as well as the UCB algorithm of \cite{improved_algorithms_for_stochastic_bandits} after a slight modification: the confidence ellipsoid is replaced by a confidence box that contains the confidence ellipsoid. Consequently, the optimism-in-the-face-of-uncertainty principle is preserved while enabling our self-correction analysis. Hence, we provide a unified framework for analyzing two fundamental algorithmic families on spherical action sets. Through the self-correction, we obtain both regret guarantees and high-probability lower bounds on the minimum eigenvalue of the design matrix, thereby characterizing their statistical inference and regret behavior, together.

{From a technical standpoint, our main contribution is an
algorithm-agnostic proof pipeline for linear bandits. The standard route to high-probability regret bounds invokes
the optimism in the face of uncertainty principle
\cite{improved_algorithms_for_stochastic_bandits,
linear_thompson_sampling_revisited}. In this route, the algorithm picks $a_t=a^\star(\tilde\theta_t)$ such that $\|\tilde \theta_t\| \geq \|\theta^\star\|$ either with high probability \cite{improved_algorithms_for_stochastic_bandits} or with a non-negligible positive probability \cite{linear_thompson_sampling_revisited}. Notice that the optimism implies that $a^\star(\tilde \theta)^\top \tilde \theta \geq a^\star(\theta^\star)^\top \theta^\star$, hence one acts as optimistically towards a better optimal pay-off. Then the per-step regret obtained by playing
$a_t=a^\star(\tilde\theta_t)$ can be bounded by 
\[
r_{\theta^\star}(a_t)
=
\underbrace{\|\theta^\star\|-\|\tilde\theta_t\|}_{\leq 0}
+
a_t^\top(\tilde\theta_t-\theta^\star) \leq a_t^\top(\tilde\theta_t-\theta^\star),
\]
where optimism makes the first term nonpositive.
To control the remaining factor, the analysis
places the chosen or sampled
parameter $\tilde\theta_t$ in ellipsoids centered at
the regularized least-squares estimate $\widehat{\theta}_t$, which also contain $\theta^\star$ with high probability, resulting in
$\|\tilde\theta_t-\theta^\star\|_{V_{t-1}}
\lesssim \sqrt{\log T}$. Thus, 
Cauchy--Schwarz and the elliptical-potential lemma, informally presented as $
\sum_{t=1}^T\|a_t\|_{V_{t-1}^{-1}}^2 \lesssim \log(T)
$
allows us to upperbound the sum of the remaining factor:
\begin{align}
 \sum_{t=1}^T \|\tilde{\theta}_t-\widehat{\theta}_t\|_{V_{t-1}}\;\|a_t\|_{V_{t-1}^{-1}}
\le
\sqrt{\Big(\sum_{t=1}^T \|\tilde{\theta}_t-\widehat{\theta}_t\|_{V_{t-1}}^2\Big)
      \Big(\sum_{t=1}^T \|a_t\|_{V_{t-1}^{-1}}^2\Big)}  \lesssim \sqrt{T} \log(T).\end{align}
}

Our analysis uses a decomposition that does not
introduce this optimism-shortfall --$\|\theta^\star\| - \|\tilde{\theta}_t\|$-- and does not invoke the
elliptical-potential lemma. The regret decomposition done in \eqref{final_intro} holds for every played action, regardless
of how it is selected. Estimation accuracy controls
the first term, while cumulative deviation from the
estimated greedy action controls the second.
Regret can therefore be bounded without invoking
optimism. Hence, \emph{algorithms designed
or traditionally analyzed through the lens of optimism admit
a common, optimism-agnostic regret analysis!} 

At the heart of our proof framework is the self-correction property in \eqref{intro:self_correction}, which captures how an algorithm directs exploration toward poorly explored directions. Related mechanisms have been identified for a substantially modified UCB algorithm, discussed in Section~\ref{sec:ucb}, and for the randomized TS-like policies studied in \cite{abeille2025when}, which resemble the linear TS policy examined in Section~\ref{sec:ts}. The former analysis relies on the specific structure of the modified UCB algorithm and does not directly extend to simpler variants or randomized policies. The latter is closer to our approach in spirit, but its $\lambda_{\min,t}$-growth guarantee is indexed by a particular subset of rounds, rather than elapsed time, and its regret analysis continues to rely on optimism.

Specifically, the policy in \cite{abeille2025when} samples a parameter $\theta_t$ from a distribution whose geometry is determined by $V_t$ and plays $a^\star(\theta_t)$. Their analysis identifies a self-correcting mechanism that increases exploration of poorly explored directions in conditional expectation. However, the corresponding self-correction inequality applies on rounds when the conditional probability of optimism is small. For the Euclidean unit-ball action set, these are rounds satisfying
$
\mathbb P\!\left(
\|\theta_t\|_2\geq\|\theta^\star\|_2
\,\middle|\,\mathcal F_{t-1}
\right)\leq p,
$
for a fixed threshold $p\in(0,1)$. If $N_t$ denotes the number of such rounds up to time $t$, their high-probability information bound gives, schematically,
$
\lambda_{\min}(V_t)\gtrsim\sqrt{N_t},
$
up to constants, logarithmic factors, and lower-order terms. On the complementary rounds, optimism occurs with conditional probability bounded away from zero. They control regret on these rounds through an optimism-based elliptical-potential argument, as in \cite{linear_thompson_sampling_revisited}, and obtain the final regret bound by combining the two regimes.

The missing piece in the literature is a single proof pipeline that treats deterministic optimism-based algorithms and randomized posterior-sampling-type policies within the same analytical framework. We obtain such a pipeline by reducing their exploration mechanisms to the same primitive object: the self-correcting evolution of $\Lambda_t$. Rather than beginning from optimism itself, our analysis first establishes a conditional self-correction inequality, converts this inequality into a time-indexed high-probability growth guarantee for $\lambda_{\min,t}$, and then uses this information growth to control the regret. The resulting argument is therefore agnostic to whether exploration is generated deterministically through optimism or randomly through parameter sampling. A related unification of deterministic and randomized linear-bandit algorithms is provided by \cite{poful}; however, its analysis does not establish an inference-to-regret pipeline of the form developed here, and instead proceeds through the elliptical-potential route.

Beyond providing a common analysis of existing algorithms, our framework also suggests a design principle for exploration. It requires only sufficient information growth together with controlled cumulative deviation from the estimated greedy action. This flexibility allows us to construct and tune exploration rules directly around these two requirements. We illustrate this point through UCB-type variants in which the set over which the OFUL principle is applied need not coincide with the usual horizon-dependent confidence set, and therefore need not be large enough to guarantee optimism with high probability. Instead, the exploration radius can be chosen to preserve the self-correction and concentrated-exploitation properties identified by our framework. The finite-horizon experiments in Section~\ref{sec:numerics} illustrate the potential benefit of this perspective: a UCB variant chosen without enforcing optimism exhibits information-growth and regret behavior remarkably similar to Thompson sampling in the experiments considered below. This observation is consistent with the well-known empirical advantage that randomized exploration can have over theoretically calibrated optimistic methods, and suggests that our framework may help narrow this practical gap by enabling the design of deterministic, optimism-agnostic exploration rules whose behavior is governed directly by the inference and exploitation requirements needed for low regret.

Concretely, we establish: (i) a modular, algorithm-agnostic framework for linear bandits over spherical action sets that decomposes regret into the cost of insufficient statistical precision, controlled by the design matrix, and the cost paid to acquire additional information; (ii) an information-growth property and a complementary exploitation property that, when combined, yield order-optimal high-probability regret bounds; (iii) that variants of UCB and Thompson sampling satisfy these two properties, providing a unified, optimism-agnostic analysis of deterministic and randomized algorithms; and (iv) that the same sufficient conditions can be used as a design principle for exploration rules that need not be calibrated to preserve optimism.

\section{A Modularized Framework for Regret Bounds}
\label{sec:framework}

In this section, we show that under certain regularity conditions on the linear bandit environment, if an algorithm satisfies two complementary properties, then it achieves optimal regret.

We begin by defining notation. For $a \in \Rset^n$ and $m > 0$, let $\Bcal_a(m)$ denote the ball of radius $m$ centered at $a$.
Let $\{\Fcal_t\}_{t=0}^T$ denote a filtration, where $\Fcal_{t-1}$ is the
$\sigma$-algebra generated by past actions and observations before time $t$, and $\Fcal_0$ is the
trivial $\sigma$-algebra. We write $\E_{t-1}[\cdot] := \E[\cdot \mid \Fcal_{t-1}]$ and $\operatorname{var}_{t-1}[\cdot] := \operatorname{var}[\cdot \mid \Fcal_{t-1}]$. For a symmetric matrix $A \in \Rset^{n \times n}$,
let $\{\lambda_i(A)\}_{i=1}^n$ be its ordered eigenvalues such that $\lambda_i(A) \leq \lambda_{i + 1}(A)$ and $\{v_i(A)\}_{i=1}^n$ the associated
orthonormal eigenbasis. Specifically, we denote $\lambda_{1}(V_t)$ as $\lambda_{\min,t}$ and $v_1(V_t)$ as $v_{\min,t}$. We use the weighted norm $\|z\|_Q := \sqrt{z^\top Q z}$ for
$z \in \Rset^n$ and $Q \succeq 0$. We record here a result that will be useful in the subsequent analysis.
\begin{lemma}
For any two vectors $x\neq0,y\neq0$, $\left\|
    \frac{x}{\|x\|}
    -
    \frac{y}{\|y\|}
\right\| \leq \frac{2\|x-y\|}{\|y\|}$.
\label{lemma:xy}
\end{lemma}
\begin{proof} The triangle inequality gives
        \begin{align}
        \begin{aligned}
\left\|
    \frac{x}{\|x\|}
    -
    \frac{y}{\|y\|}
\right\|
&\leq
\left\|
    \frac{x}{\|x\|}
    -
    \frac{x}{\|y\|}
\right\|
+
\left\|
    \frac{x-y}{\|y\|}
\right\| =
\frac{\bigl|\|y\|-\|x\|\bigr|}{\|y\|}
+
\frac{\|x-y\|}{\|y\|} \leq
\frac{2\|x-y\|}{\|y\|}.
\end{aligned}
    \end{align}
\end{proof}
With this notation in hand, we now begin by precisely stating the assumptions on the linear bandit environment.
\begin{assumption}\label{assumption_1}
The linear bandit environment satisfies:
\begin{enumerate}[leftmargin=*,nolistsep,noitemsep]
    \item Parameter bounds:
    $0 < \theta_{\min} := \inf_{\theta \in \Theta} \|\theta\| \leq
    \theta_{\max} := \sup_{\theta \in \Theta} \|\theta\| < \infty$.
    \item Noise process: $\varepsilon_t$ is i.i.d., zero-mean, and $M$-sub-Gaussian for some $M>0$, with $\E_{t-1}[Y_t\mid a_t]=\theta^{\star\top}a_t.$ 
    \item Action set: $\Acal = \{ x \in \Rset^n : \|x\| =1 \}$.
\end{enumerate}
\end{assumption}
In this bandit environment, our goal is to define two properties that guarantee a regret upper bound. Next, we delineate these properties. In the sequel, we use the notation,
\begin{align}
    M_T := \lceil  \sqrt{T}\log(T) \rceil, \quad m_T= \left\lceil \frac{1}{3} \sqrt{T}\log(T) \right\rceil.
\end{align}
We ignore the integrality of these time instances.
\begin{property}[Guaranteed inference]\label{prop_1}
Under a control policy, there exists a positive constant $C_1$ independent of $T$ such that
$\|\widehat{\theta}_{t+1} - \theta^\star\|^2 \leq {C_1\log(T)}/{ \sqrt{t}}$ for all $ t \in [M_T,T]$ w.p. at least $1-\frac{1}{2T}$ for a large enough $T$.
\end{property}
\begin{property}[Concentrated Explorations]\label{prop_2}
Under a control policy, there exists a positive constant $C_2$ independent of $T$ such that $\sum_{t = M_T + 1}^T \|a_t - a^\star(\widehat{\theta}_t)\|^2
\leq   C_2\sqrt{T}\log(T)$ w.p. at least $1-\frac{1}{2T}$ for a  large enough $T$.
\end{property}
Properties~\ref{prop_1} and~\ref{prop_2} lead to a high-probability order-optimal regret accumulation in the linear bandit environment of Assumption \ref{assumption_1}, as stated next.
\begin{theorem}\label{theorem_unified_regret_bound}
Suppose Assumption \ref{assumption_1} holds. Then, for any control policy that satisfies Properties \ref{prop_1} and \ref{prop_2}, there exists a constant $C_3 > 0$ independent of $T$  such that
\begin{align}
\probb{\mathscr{R}_{\theta^{\star}}(T) \leq C_3 \sqrt{T}\log(T)} \geq 1 -  \frac{1}{T}
\end{align}
for a large enough $T$.
\end{theorem}
\begin{proof}
Consider the event in which Properties ~\ref{prop_1} and
\ref{prop_2} hold. Union bound yields
that it holds with probability at least $1-\frac{1}{T}$. We work on this high-probability event. From \eqref{eq:spherical_regret_geometry}, we have that for every $a\in\Acal$,
\begin{align}
\begin{aligned}
r_\theta(a) = 
\frac{\|\theta\|}{2}
\|a^\star(\theta)-a\|^2,
\end{aligned}
\end{align}
which further implies that $r_{\theta^\star}(a)\leq 2\theta_{\max}$. Therefore,
\begin{align} \label{first+_reg_decomp}\mathscr R_{\theta^\star}(T)
\leq
2\theta_{\max} M_T
+
\frac{\theta_{\max}}{2}
\sum_{t=M_T+1}^T
\|a_t-a^\star(\theta^\star)\|^2.
\end{align}

We bound the sum using the Lipschitz continuity of $a^\star$.  Lemma \ref{lemma:xy} implies that when $\widehat \theta_t \neq 0$, we have
\begin{align}\label{lip}
\|a^\star(\widehat\theta_t)-a^\star(\theta^\star)\|
\leq \frac{2}{\|\theta^\star\|}
\|\widehat\theta_t-\theta^\star\| \leq 
\frac{2}{\theta_{\min}}
\|\widehat\theta_t-\theta^\star\|.
\end{align}

Property \ref{prop_1} gives
$
\|\widehat\theta_t-\theta^\star\|^2
\lesssim
\log(T)/\sqrt{t-1}$.
Thus, for all sufficiently large $T$ and $t \geq M_T+1$ we have
$\|\widehat\theta_t-\theta^\star\|<\theta_{\min}/2$ and hence
$\widehat\theta_t\neq0$. Thus, \eqref{lip} holds for all $t \geq M_T+1$ which when combined with $\|a_t-a^\star(\theta^\star)\|^2
\leq
2\|a_t-a^\star(\widehat\theta_t)\|^2
+
2\|a^\star(\widehat\theta_t)-a^\star(\theta^\star)\|^2$ and \eqref{first+_reg_decomp} gives
\begin{align}
     \begin{aligned}
\mathscr R_{\theta^\star}(T)
\leq
2\theta_{\max}M_T
+
\theta_{\max}
\sum_{t=M_T+1}^T
\|a_t-a^\star(\widehat\theta_t)\|^2 + 
\frac{4\theta_{\max}}{\theta_{\min}^2}
\sum_{t=M_T+1}^T
\|\widehat\theta_t-\theta^\star\|^2,
\end{aligned}\label{regret_proof_decomposition}
\end{align}
where the first summation is upper bounded via Property \ref{prop_2} and the second summation is upper bounded via Property \ref{prop_1} and an integral bound 
\begin{align}
\begin{aligned}
\sum_{t=M_T+1}^T
\|\widehat\theta_t-\theta^\star\|^2
\leq
C_1\log(T)
\sum_{t=M_T+1}^T\frac{1}{\sqrt{t-1}}\leq
2C_1\sqrt{T}\log(T).
\end{aligned}
\end{align}
Since $M_T \lesssim \sqrt{T}\log(T)$, all terms in \eqref{regret_proof_decomposition} are of order $\mathcal{O}(\sqrt{T}\log(T))$, which completes the proof.
\end{proof}

\section{The Inference Dynamics}
\label{sec:guaranteed_inference}

We have identified two properties thus far that are sufficient for
order-optimal regret. We now give a condition on the evolution of a certain statistic derived from the
design matrix that implies Property~\ref{prop_1} and is amenable to
verification for algorithms such as linear TS and UCB. Recall that $V_t=\lambda I+\sum_{s=1}^t a_sa_s^\top$ and
$\Lambda_t:=1/\operatorname{tr}(V_t^{-1})$. We use the following self-normalized concentration result from
\cite[Theorem~2]{improved_algorithms_for_stochastic_bandits} after modifying into our
notation.  It states that under Assumption~\ref{assumption_1}, for any $\delta\in(0,1)$,
with probability at least $1-\delta$, for all $t\geq 1$,
\begin{align}
\left\|\widehat\theta_{t+1}-\theta^\star\right\|_{V_t}
\leq
M\sqrt{
    n\log\left(
        \frac{1+t/\lambda}{\delta}
    \right)
}
+
\sqrt{\lambda}\,\theta_{\max}
\label{eq:self_normalized_AY}
\end{align}
for any sequence of bounded actions.
With $\delta=1/(4T)$ and using $t\leq T$ with
$\left\|\widehat\theta_{t+1}-\theta^\star\right\|_{V_t}^2
\geq
\lambda_{\min,t}
\left\|\widehat\theta_{t+1}-\theta^\star\right\|^2$
gives
\begin{align}
\label{eq:self_normalized_concentration}
\probb{
    \vnorm{\widehat{\theta}_{t+1}-\theta^\star}^2
    \leq
    \rho_T^2\lambda_{\min,t}^{-1}, \; t = 1,\ldots, T
}
\geq
1-\frac{1}{4T},
\end{align}
where 
\begin{align}
\rho_T
:=
M\sqrt{
    n\log\left(
        4T(1+T/\lambda)
    \right)
}
+
\sqrt{\lambda}\,\theta_{\max}
=
\mathcal O(\sqrt{\log T}).
\label{eq:rhoT.def}
\end{align}
Since $\Lambda_t\leq\lambda_{\min,t}$, a lower bound of the form
$\Lambda_t\gtrsim\sqrt t$ is sufficient to obtain a guaranteed inference quality of
$\vnorm{\widehat{\theta}_{t+1}-\theta^\star}^2
\lesssim\log(T)/\sqrt t$, as required by Property~\ref{prop_1}.
We therefore ask when a policy generates such growth. A path-wise
condition of the form
$\Lambda_t-\Lambda_{t-1}\gtrsim\Lambda_{t-1}^{-1}$ would suffice, since
it produces constant-order growth of $\Lambda_t^2$. We impose a weaker counterpart that the same holds \emph{in expectation}.
\begin{property}[Self-Correction for Inference]
\label{prop_3}
Under a control policy, 
\begin{align}
\E\!\left[\Lambda_t\mid\Fcal_{t-1}\right] \mathbb{I}_{\mathcal{G}_{t-1}}
&\geq
\left(\Lambda_{t-1}
+
\frac{C_4}{\Lambda_{t-1}}\right)\mathbb{I}_{\mathcal{G}_{t-1}}, \quad t=m_T+1,\ldots,T,
\label{eq:prop3_drift}
\end{align}
where 
$
\mathcal G_{m_T}
\supseteq
\mathcal G_{m_T+1}
\supseteq
\cdots
\supseteq
\mathcal G_{T-1},
$ with $\mathcal G_{t-1}\in\mathcal F_{t-1}$, $t=m_T+1,\ldots,T$, and 
$\probb{\mathcal G_{T-1}
}
\geq
1-\frac{1}{8T}$, for large $T$ for some $T$-independent positive constant
$C_4$.
\end{property}

If the inequality holds almost surely, one may take
$\mathcal G_{t-1}=\Omega$.  If not, our proof requires that the growth condition on $\Lambda_t$ holds ``nearly'' always, making precise how close to unity is sufficient. 
It only controls the rate at which information
accumulates under the policy. We call $\Lambda_t$ self-correcting in that its per-step growth is controlled by both $\Lambda_{t-1}$ and its reciprocal. At least one of them is large enough to fuel the per-step growth in $\Lambda_t$. Property~\ref{prop_3} does not require $a_t$ to be close to
$a^\star(\widehat{\theta}_t)$ and therefore does not imply
Property~\ref{prop_2}.

In our proofs for linear TS and UCB, we take $\mathcal G_{t-1}$ to be the event that the least-squares estimator remains in a \emph{fixed} neighborhood of $\theta^\star$,
\begin{align}\label{g_tdefined_1}
\mathcal G_{t-1}
:=
\left\{
\|\widehat\theta_s-\theta^\star\|
\leq
\frac{\theta_{\min}}{2},
\quad s = m_T+1, \ldots, t
\right\}.
\end{align}
Notice that $\lambda_{\min, m_T}\geq(\log T)^2$ together with $V_t\succeq V_{m_T}$ for $t\geq m_T + 1$, implies
\begin{align}\label{gt_defined_2}
\left\|\widehat\theta_{t+1}-\theta^\star\right\|_{V_t}^2
\geq
\lambda_{\min,t}
\left\|\widehat\theta_{t+1}-\theta^\star\right\|^2
\geq
(\log T)^2
\left\|\widehat\theta_{t+1}-\theta^\star\right\|^2.
\end{align}
Hence, \eqref{eq:self_normalized_AY} with $\delta = 1/(8T)$ gives  
\begin{align}\label{g_tdefined_3}
\probb{\mathcal G_{T - 1} } \geq 1 - \frac{1}{8T},
\end{align}
whenever $\lambda_{\min,m_T}\geq(\log T)^2$ for a large enough $T$, which is a mild requirement typically satisfied by algorithms, or can be explicitly encoded through pure exploration over a burn-in period of length $m_T = o(T)$; see Section \ref{sec:ts}. Our analyses of linear TS and UCB will proceed over this good event. The next result captures how self-correction in Property~\ref{prop_3} produces the desired
$\sqrt t$ growth uniformly after the burn-in period.

\begin{theorem}
\label{theorem_unified_inference_bound}
Suppose Assumption~\ref{assumption_1} holds. If a control policy 
satisfies Property~\ref{prop_3}, then for sufficiently large $T$,
\begin{align}
\mathbb P\left\{
\lambda_{\min,t}\geq\Lambda_t
\geq\frac12\sqrt{C_4t},
\quad t = M_T,\ldots,T\}
\right\}
\geq 1-\frac1{4T},
\label{eq:thm.inference}
\end{align}
and the policy satisfies Property~\ref{prop_1}.
\end{theorem}
{\begin{proof}
The inequality $\lambda_{\min,t} \geq \Lambda_t$ always holds. For the rest, fix $t \in [M_T, T]$ 
and define 
\begin{align}
    D_s &:=\Lambda_s-\Lambda_{s-1},
    \\
    J_s &:= \mathbb{I}_{\{\Lambda_{s-1} < \frac{1}{2}\sqrt{C_4 t}\}} \mathbb{I}_{\mathcal{G}_{s-1}},
    \\
    W_s &:=(D_s - \mathbb{E}_{s-1}[D_s]) J_s
\end{align}
for $s=m_T+1, \ldots, t$. Then, $W$'s form a martingale difference sequence adapted to $\{\mathcal{F}_s\}_{s \geq m_T}$.  Using martingale concentration on $W$, we will show that 
\begin{align}
    \probb{ \underbrace{ \Lambda_t < \frac{1}{2}\sqrt{C_4 t}  \text{ and } \mathcal{G}_{t-1} }_{:=\mathcal{H}_t}}  
    \leq 
    \exp \left(- \frac{1}{8} \sqrt{C_4 t}\right)
    \label{eq:reqd.3.2}
\end{align}
for $t = M_T, \ldots, T$ and a large enough $T$. Since $\mathcal{G}_{ T-1} \subseteq \mathcal{G}_{t-1}$ for all $t \leq T$, a union bound over $t \in [M_T, T]$ using the above relation yields
\begin{align}
\begin{aligned}
    \probb{\exists t \in [M_T, T] : \Lambda_t < \frac{1}{2}\sqrt{C_4 t}} 
    &\leq 
    \probb{\mathcal{G}_{T-1}^c} + \sum_{t=M_T}^T \probb{\Lambda_t < \frac{1}{2}\sqrt{C_4 t} \text{ and } \mathcal{G}_{T-1}} \\
    &\leq 
    \probb{\mathcal{G}_{T-1}^c} + \sum_{t=M_T}^T \probb{\mathcal{H}_t} \\
    &\leq \frac{1}{8T} + T \exp \left(- \frac{1}{8} \sqrt{C_4 M_T}\right) \le \frac{1}{4T} 
\end{aligned}
\end{align}
for sufficiently large $T$, because $M_T \asymp \sqrt{T} \log(T)$. This implies \eqref{eq:thm.inference}. Property \ref{prop_1} follows from \eqref{eq:thm.inference} via the following argument.

From \eqref{eq:rhoT.def}, $\rho_T^2\leq c_\rho\log T$ for a $T$-independent constant $c_\rho > 0$ for large $T$. On the event in \eqref{eq:thm.inference}, 
$\lambda_{\min,t}\geq\frac{1}{2}\sqrt{C_4 t}$ for every
$t=M_T,\ldots,T$, and it has probability of at least $1-\frac{1}{4T}$. On the event
in \eqref{eq:self_normalized_concentration}, $\vnorm{\widehat{\theta}_{t+1}-\theta^\star}^2
\leq
\rho_T^2\lambda_{\min,t}^{-1}$ $t=M_T,\ldots,T$ and it has a probability of at least $1-\frac{1}{4T}$. Hence, on their intersection, we have
\begin{align}
\begin{aligned}
\vnorm{\widehat{\theta}_{t+1}-\theta^\star}^2
\leq
\rho_T^2\lambda_{\min,t}^{-1}\leq
\frac{2c_\rho}{\sqrt{C_4}}
\frac{\log T}{\sqrt t},
\end{aligned}
\end{align}
which has probability of at least $1-\frac{1}{2T}$.
Therefore, Property \ref{prop_1} holds with
$C_1:=2c_\rho/\sqrt{C_4}$.

It remains to establish \eqref{eq:reqd.3.2}, which we accomplish using the following steps:
\begin{enumerate}
    \item We relate the left-hand side of \eqref{eq:reqd.3.2} to $W$'s and show that 
    \begin{align}
        \probb{ \mathcal{H}_t} 
        \leq \probb{-\sum_{s=m_T+1}^t W_s \geq \frac{3}{4} \sqrt{C_4 t}}.
        \label{eq:Et.Ws}
    \end{align}
    \item Next, we use a time-uniform Freedman inequality to bound the right-hand-side of \eqref{eq:Et.Ws} with the right-hand-side of \eqref{eq:reqd.3.2}.
\end{enumerate}

\noindent $\bullet$ \emph{Step 1. Proving \eqref{eq:Et.Ws}.}

Fix $t \in [M_T, T]$. In Loewner order, $V_{s-1} \leq V_{s-1} + a_s a_s^\top = V_{s}$. The map $A\mapsto 1/\operatorname{tr}(A^{-1})$ is monotone in that order, and hence, $\Lambda_s$ is non-decreasing, implying $D_s \geq 0$. Consider the event $\mathcal{H}_t$. Because $\Lambda_s$ is non-decreasing, $\Lambda_{s-1} \leq \Lambda_t < \frac{1}{2}\sqrt{C_4 t}$ for all $s \in \{m_T+1, \dots, t\}$. Thus, $J_s = 1$ identically on this event. On $\mathcal{H}_t$, we have
\begin{align}
    \Lambda_t - \Lambda_{m_T} = \sum_{s=m_T+1}^t W_s
    + \sum_{s=m_T+1}^t \E_{s-1}[D_s].
    \label{eq:W.telescope}
\end{align}
Since $\mathcal{H}_t\subseteq \mathcal{G}_{ t-1}$,
Property~\ref{prop_3} implies that for every $s=m_T+1,\ldots,t$,
$\mathbb{E}_{s-1}[D_s]\geq C_4/\Lambda_{s-1} \geq C_4/\Lambda_{t}$.
Also, $m_T/M_T \approx 1/3$, and hence, $m_T\leq 3M_T/8\leq 3t/8$ for a large enough $T$. These observations taken together implies the following lower bound on $\sum_{s=m_T+1}^t \E_{s-1}[D_s]$ on $\mathcal{H}_t$.
\begin{align}
    \sum_{s=m_T+1}^t \E_{s-1}[D_s]
    \geq (t - m_T)\frac{C_4}{\frac{1}{2}\sqrt{C_4 t}}
    \geq \frac{5}{8}t \cdot \frac{2\sqrt{C_4}}{\sqrt{t}}
    = \frac{5}{4}\sqrt{C_4 t},
\end{align}
Consequently, \eqref{eq:W.telescope} yields
\begin{align}
-\sum_{s=m_T+1}^t W_s
&= \sum_{s=m_T+1}^t \E_{s-1}[D_s] + \Lambda_{m_T} -     \Lambda_t
\geq
\frac{5}{4}\sqrt{C_4 t} - \frac{1}{2}\sqrt{C_4 t} = \frac{3}{4}\sqrt{C_4 t},
\end{align}
proving \eqref{eq:Et.Ws}.

\noindent $\bullet$ \emph{Step 2. Martingale Concentration of $\sum_{s=m_T+1}^t W_s$.}
We start by bounding $|W_s|$. We have
\begin{align}
     V_s = V_{s-1} + a_s a_s^\top  \leq V_{s-1} + \vnorm{a_s}^2 I.
\end{align}
The scalar function $g(x) := 1/\operatorname{tr}((V_{s-1} + x I)^{-1})$ is increasing in $x \geq 0$ with $g'(x) \leq 1$, meaning that 
$D_s\leq g(\vnorm{a_s}^2)-g(0)\leq\vnorm{a_s}^2 \leq 1$, owing to
Assumption~\ref{assumption_1}. Thus, we get $0\leq D_s \leq 1$. The same bounds also apply to $| W_s | \leq | D_s - \mathbb{E}_{s-1}[D_s] |$
almost surely. Then, by the maximal form of Freedman's inequality
\cite[Theorem~1.1]{tropp2011freedman},
\begin{align}\label{freedman}
\probb{
\exists\, k\in\{m_T+1,\ldots,t\}:
-\sum_{s=m_T+1}^k W_s\geq u_1,\;
\sum_{s=m_T+1}^k
\mathbb E_{s-1}[W_s^2]\leq u_2
}
\leq
\exp\left(
-\frac{u_1^2/2}{u_2+u_1/3}
\right),
\end{align}
for any $u_1 \geq 0, u_2 \geq 0$. Now, we show that for
$ u_1=\frac{3}{4}\sqrt{C_4 t}$ and $
u_2=\frac54\sqrt{C_4t}+2,
$
\begin{align}
\left\{
-\sum_{s=m_T+1}^t W_s\geq u_1
\right\}
\subseteq
\Bigg\{
\exists\, k\in\{m_T+1,\ldots,t\}:
-\sum_{s=m_T+1}^k W_s\geq u_1,\;
\sum_{s=m_T+1}^k
\mathbb{E}_{s-1}[W_s^2]\leq u_2
\Bigg\}.
\label{eq:freedman.inclusion}
\end{align}
Once this is established, \eqref{freedman} implies 
\begin{align}
\probb{
-\sum_{s=m_T+1}^t W_s
\geq
\frac34\sqrt{C_4t}
}
\leq
\exp\left(
-\frac{
\frac12\left(\frac34\sqrt{C_4t}\right)^2
}{
\frac54\sqrt{C_4t}+2
+
\frac13\left(\frac34\sqrt{C_4t}\right)
}
\right)\leq
\exp\left(
-\frac18\sqrt{C_4t}
\right),
\end{align}
for sufficiently large $m_T$, and hence for sufficiently large $T$. This bounds the probability given in \eqref{eq:Et.Ws}.

To prove \eqref{eq:freedman.inclusion}, consider any sample path such that
$
-\sum_{s=m_T+1}^t W_s\geq u_1.
$
Then there exists a first index
$
\tau'
:=
\min\left\{
k\in\{m_T+1,\ldots,t\}:
-\sum_{s=m_T+1}^k W_s\geq u_1
\right\}.
$
By the minimality of $\tau'$,
$
-\sum_{s=m_T+1}^{\tau'-1}W_s<u_1,
$
where the sum is understood to be zero if $\tau'=m_T+1$. Since
$|W_\tau'|\leq 1$, we have
\begin{align}
u_1 \leq -\sum_{s=m_T+1}^{\tau'}W_s
<
u_1+1.
\label{eq:first.crossing}
\end{align}
From the definition of $W_s$,
\begin{align}
-\sum_{s=m_T+1}^{\tau'} W_s
=
\sum_{s=m_T+1}^{\tau'} J_s\mathbb{E}_{s-1}[D_s]-\sum_{s=m_T+1}^{\tau'} D_sJ_s.
\label{eq:Ak.Rk}
\end{align}
Suppose that, the following two equations hold:
\begin{align}
&\sum_{s=m_T+1}^{\tau'} D_sJ_s
\leq
\frac12\sqrt{C_4t}+1,
\label{eq:Rk.bound}\\
&\sum_{s=m_T+1}^{\tau'}
\mathbb{E}_{s-1}[W_s^2]
\leq
\sum_{s=m_T+1}^{\tau'} J_s\mathbb{E}_{s-1}[D_s].
\label{eq:variance.Ak}
\end{align}

Then,
\begin{align}
\begin{aligned}
\sum_{s=m_T+1}^{\tau'} J_s\mathbb{E}_{s-1}[D_s]
&=
-\sum_{s=m_T+1}^{\tau'}W_s
+\sum_{s=m_T+1}^{\tau'} D_sJ_s 
\\
&\leq
\frac34\sqrt{C_4t}
+
\frac12\sqrt{C_4t}
+2 = 
u_2.
\label{eq:variance.bound.tau}
\end{aligned}
\end{align}
where the first line follows from \eqref{eq:Ak.Rk}, and the second line follows from \eqref{eq:first.crossing} and \eqref{eq:Rk.bound}. Hence, from \eqref{eq:variance.Ak},
\begin{align}
\sum_{s=m_T+1}^{\tau'}
\mathbb{E}_{s-1}[W_s^2]
\leq
\sum_{s=m_T+1}^{\tau'} J_s\mathbb{E}_{s-1}[D_s]
<
u_2.
\end{align}
Thus, at  $\tau'$,
$
-\sum_{s=m_T+1}^{\tau'}W_s\geq u_1 $ and $
\sum_{s=m_T+1}^{\tau'}
\mathbb{E}_{s-1}[W_s^2]\leq u_2, $
which proves \eqref{eq:freedman.inclusion}. Therefore, what remains to do is proving \eqref{eq:Rk.bound} and \eqref{eq:variance.Ak}. 

To that end, notice that $J_s$ is
$\mathcal F_{s-1}$-measurable, and that $D_s^2\leq D_s$ almost surely. These observations imply
\begin{align}
    \sum_{s=m_T+1}^{\tau'} \mathbb{E}_{s-1}[W_s^2]
    = \sum_{s=m_T+1}^{\tau'} J_s \mathrm{Var}_{s-1}(D_s)\leq
\sum_{s=m_T+1}^{\tau'} J_s \mathbb{E}_{s-1}[D_s^2]
    \leq \sum_{s=m_T+1}^{\tau'} J_s \mathbb{E}_{s-1}[D_s].
\end{align}
This proves \eqref{eq:variance.Ak}.

Now, we prove \eqref{eq:Rk.bound}. Notice that 
$ 
J_s\leq \mathbb{I}_{\{\Lambda_{s-1}<\frac12\sqrt{C_4t}\}}. 
$ 
Thus, if 
$ 
\Lambda_{m_T}\geq \frac12\sqrt{C_4t}, 
$ 
then, by the monotonicity of $\Lambda_s$, 
\[
\sum_{s=m_T+1}^{\tau'}D_sJ_s=0. 
\]
Hence, \eqref{eq:Rk.bound} holds trivially. Now consider the case where 
$ 
\Lambda_{m_T}<\frac12\sqrt{C_4t}. 
$ 
Let 
$ 
\tau:=\inf\left\{ 
s\geq m_T+1: 
\Lambda_s\geq\frac12\sqrt{C_4t} 
\right\}. 
$ 
Then 
\begin{align} 
\sum_{s=m_T+1}^{\tau'}D_sJ_s 
&\leq 
\sum_{s=m_T+1}^{\tau'} 
D_s\mathbb{I}_{\{\Lambda_{s-1}<\frac12\sqrt{C_4t}\}} 
\leq 
\Lambda_{\min\{\tau',\tau\}}-\Lambda_{m_T}. 
\label{eq:DsJs.bound} 
\end{align} 
Since $\tau'\leq t$, we have 
$ 
\min\{\tau,\tau'\}\leq \min\{\tau,t\}. 
$ 
By the monotonicity of $\Lambda_s$, 
$ 
\Lambda_{\min\{\tau,\tau'\}} 
\leq 
\Lambda_{\min\{\tau,t\}}. 
$  If $\tau\leq t$, then by the definition of $\tau$, 
$
\Lambda_{\tau-1}<\frac12\sqrt{C_4t}, 
$ 
and therefore, since $D_\tau\leq1$, 
\[ 
\Lambda_{\min\{\tau,t\}} 
= 
\Lambda_\tau 
= 
\Lambda_{\tau-1}+D_\tau 
\leq 
\frac12\sqrt{C_4t}+1. 
\] 
On the other hand, if $\tau>t$, then the threshold has not been reached by time $t$, so 
\[ 
\Lambda_{\min\{\tau,t\}} 
= 
\Lambda_t 
< 
\frac12\sqrt{C_4t}. 
\] 
Thus, in either case, 
$ 
\Lambda_{\min\{\tau,\tau'\}} 
\leq 
\frac12\sqrt{C_4t}+1. 
$ 
Combining this with \eqref{eq:DsJs.bound} and using 
$\Lambda_{m_T}\geq0$, we obtain 
$ 
\sum_{s=m_T+1}^{\tau'}D_sJ_s 
\leq 
\frac12\sqrt{C_4t}+1, 
$ 
which proves \eqref{eq:Rk.bound}.
\end{proof}}

The result bears resemblance to \cite[Lemma 9]{abeille2025when} in the context of linear TS. The preceding theorem turns 
Property \ref{prop_3} into the inference guarantee given in Property \ref{prop_1} which is used in the
regret analysis.

The pipeline for optimal regret accumulation therefore is now reduced to the diagram shown below.
\begin{figure}[H]
\centering
\begin{tikzpicture}[
  node distance = 1.4cm and 1.6cm,
  box/.style = {
    rectangle,
    rounded corners = 6pt,
    draw,
    minimum width = 5.4cm,
    minimum height = 2.4cm,
    align = center,
    font = \small,
    inner sep = 10pt,
    line width = 0.4pt
  },
  purplebox/.style = {box, fill = purple!8,  draw = purple!50},
  tealbox/.style  = {box, fill = teal!8,    draw = teal!50},
  coralbox/.style = {box, fill = orange!8,   draw = orange!50},
  optbox/.style   = {box, fill = teal!8,    draw = teal!50,
                     minimum width = 4.8cm,  minimum height = 1.6cm},
  midlabel/.style = {
    rectangle, rounded corners = 3pt,
    draw = gray!60, fill = white, text = black!80,
    font = \footnotesize\bfseries, inner sep = 4pt, line width = 0.4pt
  },
  arr/.style = {
    -{Stealth[length=8pt, width=7pt]},
    line width = 2pt, gray!60
  },
  connector/.style = {line width = 2pt, gray!60},
  arrBold/.style = {
    -{Stealth[length=8pt, width=7pt]},
    line width = 2pt, gray!60
  }
]

\node[purplebox] (P31) {
  \textbf{Property 3.1} \\[4pt]
  Self-correction \\[2pt]
  \footnotesize
  $\mathbb{E}[\Lambda_t \mid \mathcal{F}_{t-1}]\mathbf{1}_{G_{t-1}} \geq
    \!\left(\Lambda_{t-1} + \dfrac{C_4}{\Lambda_{t-1}}\right)\!\mathbf{1}_{G_{t-1}}$
};

\node[tealbox, below = of P31] (P23) {
  \textbf{Property 2.3} \\[4pt]
  Inference guarantee \\[2pt]
  \footnotesize
  $\|\widehat{\theta}_{t+1} - \theta^\star\|^2 \leq C_1\dfrac{\log T}{\sqrt{t}}$
};
 
\node[coralbox, right = of P23] (P24) {
  \textbf{Property 2.4} \\[4pt]
  Concentrated exploitation \\[2pt]
  \footnotesize
  $\displaystyle\sum_{t=M_T+1}^{T}
    \|a_t - a^\star(\widehat{\theta}_t)\|^2 \leq C_2\sqrt{T}\log T$
};

\coordinate (midTop) at ($(P23.south)!0.5!(P24.south)$);
\node[optbox, below = 3.4cm of midTop] (OPT) {
  \textbf{Optimal regret} \\[4pt]
  \footnotesize
  $\mathscr{R}_{\theta^\star}(T) \leq C_3\sqrt{T}\log T \quad \text{w.h.p.}$
};

\draw[arrBold] (P31.south) --
  node[midlabel, midway]{Theorem~3.2}
  (P23.north);

\coordinate (mergeM) at ($(OPT.north) + (0, 1.0)$);

\draw[connector] (P23.south) -- (mergeM);

\draw[connector] (P24.south) -- (mergeM);

\draw[arr] (mergeM) --
  node[midlabel, midway]{Theorem~2.5}
  (OPT.north);
 
\end{tikzpicture}
\caption{A diagramatic representation of the proof structure for optimal regret accumulation in linear bandits with spherical action sets.}
 \end{figure}
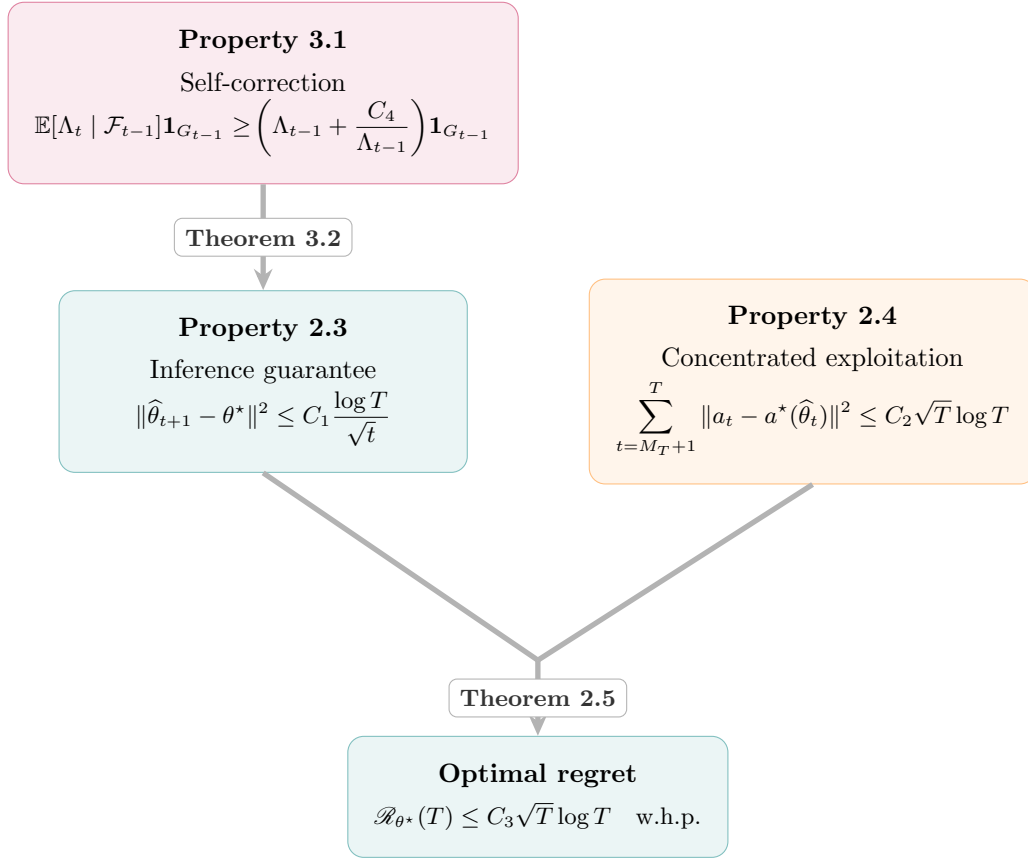

In the next two sections, we tackle the regret analysis of linear TS and UCB through the lens of the aforementioned pipeline. In particular, for both algorithms, we consider baseline good events  $\mathcal{G}_{m_T}, \ldots, \mathcal{G}_{T-1}$ that offer a uniform (loose) bound on the estimation error between times $m_T+1$ and $T$. Then, we verify the conditional drift condition in Property \ref{prop_3} and the sufficient exploitation condition of Property \ref{prop_2} on these baseline events. The known order-optimal regret accumulation rate for these algorithms then follow from Theorem
\ref{theorem_unified_regret_bound}, illustrating the power of the analytical pipeline that we have developed.

The analysis of both algorithms will rely on the structure of the drift $\Lambda_t - \Lambda_{t-1}$. Since $V_t$'s are monotonically non-decreasing in the Loewner order, we get,  
\begin{align}
\begin{aligned}
\Lambda_t-\Lambda_{t-1}
&=
\frac{
\operatorname{tr}(V_{t-1}^{-1})-\operatorname{tr}(V_t^{-1})
}{
\operatorname{tr}(V_t^{-1})\operatorname{tr}(V_{t-1}^{-1})
}
\\
&\geq
\Lambda_{t-1}^2
\left(
\operatorname{tr}(V_{t-1}^{-1})
-
\operatorname{tr}(V_t^{-1})
\right)
    \\
    &=
    {\Lambda_{t-1}^2}
    \frac{
    a_t^\top V_{t-1}^{-2}a_t
    }{
    1+a_t^\top V_{t-1}^{-1}a_t
    }
    \\
    &\geq \frac{\Lambda_{t-1}^2}{1+\lambda^{-1}} a_t^\top V_{t-1}^{-2}a_t.
    \end{aligned}
    \label{eq:Lt.Lt1}
\end{align}
The second-last line follows from the Sherman-Morrison formula for a rank-one update, and the last line from $V_{t-1} \succeq \lambda I$. As will become clear, the art lies in bounding $a_t^\top V_{t-1}^{-2}a_t$ from below in expectation for linear TS and path-wise for UCB with an order of $\Lambda_{t-1}^{-3}$.

The two algorithms take very different routes in achieving this bound. UCB's choice of $a_t$ is $\mathcal{F}_{t-1}$-measurable, and this deterministic choice allows us to relate $a_t^\top V_{t-1}^{-2}a_t$ to $\Lambda_{t-1}^{-3}$ directly, and not just in expectation as required in \eqref{eq:prop3_drift}. Such a guarantee is possible, given the convenient closed-form expression of a confidence set examined in Section \ref{sec:ucb}. However, such a path-wise guarantee is untenable, given the randomness in linear TS examined in Section \ref{sec:ts}. In contrast, to examine linear TS, $\E_{t-1}[a_t^\top V_{t-1}^{-2}a_t]$ is transformed to $\operatorname{tr}(V_{t-1}^{-2} \E_{t-1}[a_t a_t^\top])$, where we invoke Lemma \ref{lem:TS-predictable-design} to lower-bound $\E_{t-1}[a_t a_t^\top]$ in the Loewner order by a scaled version of $V_{t-1}^{-1}$. One might ask if this lemma holds for UCB. Notice that, $a_t a_t^\top$ when conditioned on $\mathcal F_{t-1}$, is a rank-one positive semidefinite matrix that cannot dominate the positive definite $V_{t-1}^{-1}$. This discussion highlights how two very different paths lead to self-correction, the underlying mechanism that produces necessary inference for regret accumulation.

\section{Analysis of Linear Thompson Sampling}
\label{sec:ts}
We now verify Properties~\ref{prop_3} and~\ref{prop_2} for linear
Thompson sampling in Algorithm~\ref{alg:thompson_sampling} on the
spherical action set. At round $t$, the algorithm samples
\[
    \widetilde\theta_t
    =
    \widehat\theta_t
    +
    c_{\mathrm{TS}} V_{t-1}^{-1/2}\eta_t
\]
and plays $a_t=a^\star(\widetilde\theta_t)$, where the perturbation
distribution is specified in Assumption~\ref{assumption:TS-perturbations}.

\begin{algorithm}
\caption{Thompson Sampling, adapted from
\cite{linear_thompson_sampling_revisited}}
\label{alg:thompson_sampling}
\begin{algorithmic}[1]
\STATE \textbf{Data:} $\widehat\theta_1$, $T$, $\lambda^{\textrm{TS}}$,
$c_{\mathrm{TS}}$, $\nu$
\STATE $V_0\gets\lambda^{\textrm{TS}} I$
\FOR{$t=1,\ldots,T$}
    \STATE Sample $\eta_t\sim\nu$
    \STATE
    $\widetilde\theta_t
    \gets
    \widehat\theta_t
    +
    c_{\mathrm{TS}} V_{t-1}^{-1/2}\eta_t$
    \STATE $a_t\gets a^\star(\widetilde\theta_t)$
    \STATE Observe $Y_t$
    \STATE $V_t\gets V_{t-1}+a_ta_t^\top$
    \STATE
    $\widehat\theta_{t+1}
    \gets
    V_t^{-1}\sum_{s=1}^t a_sY_s$
\ENDFOR
\end{algorithmic}
\end{algorithm}

\begin{assumption}[TS perturbations]
\label{assumption:TS-perturbations}
The perturbations $(\eta_t)_{t\geq1}$ are independent across time,
and $\eta_t$ is independent of $\Fcal_{t-1}$ with a common distribution that satisfies $\eta_t\overset{d}=-\eta_t$,  $\E\!\left[
            \eta_t\eta_t^\top
            \mathbbm 1_{\{\|\eta_t\|\leq L_\eta\}}
        \right] \succeq
        \kappa_\eta I_n$ for some $L_\eta<\infty$ and $\kappa_\eta>0$
    independent of $t$ and $T$, $\mathbb{P}(\widetilde\theta_t=0|\mathcal F_{t-1})=0$ almost surely for
    every $t$, and for every $u\in\mathbb S^{n-1}$, $\E[(u^\top\eta_t)^2]\leq 1$ and $\E[(u^\top\eta_t)^4]\leq K_\eta$.
\end{assumption}

Gaussian perturbations satisfy
Assumption~\ref{assumption:TS-perturbations}. For
$\eta_t\sim\mathcal N(0,I_n)$, one may take
$L_\eta=\sqrt{2(n+2)}$ and $\kappa_\eta=1/2$.
The conditions in Assumption~\ref{assumption:TS-perturbations} encode the same balance between anti-concentration and concentration of linear Thompson sampling distributions emphasized in \cite{linear_thompson_sampling_revisited}. The truncated second-moment condition
$\E[\eta_t\eta_t^\top\mathbbm 1_{{\|\eta_t\|\leq L_\eta}}]\succeq\kappa_\eta I_n$
plays the role of anti-concentration where among perturbations of moderate size, there remains nontrivial variation in every direction, which is what allows us to produce a lower bound on the predictable design matrix. Additionally, the $2^{\text{nd}}$ and $4^{\text{th}}$ moment bounds on $\eta_t$ prevent the perturbations from becoming too large, thereby controlling the cost of exploration.
The next result documents a key observation that
$\E_{t-1}[a_ta_t^\top]$ dominates a constant multiple of
$V_{t-1}^{-1}$ in the Loewner order, a fact that plays a crucial role in establishing Properties \ref{prop_3}
and \ref{prop_2}.

\begin{lemma}
\label{lem:TS-predictable-design}
Suppose Assumption~\ref{assumption:TS-perturbations} holds. If there exists $\Gamma > 0$ such that
$\|\widehat\theta_t\|\leq \Gamma$ and
$ \lambda_{\min,t-1} \geq 1$, then
\begin{align}
    \E_{t-1}[a_ta_t^\top]
    \succeq
    \frac{
        \kappa_\eta c_{\mathrm{TS}}^2
    }{
        (\Gamma+ c_{\mathrm{TS}}L_\eta)^2
    }
    V_{t-1}^{-1}.
\end{align}
\end{lemma}

\begin{proof}
Let
$Z_t:=c_{\mathrm{TS}}V_{t-1}^{-1/2}\eta_t$. Then,
$\widetilde\theta_t=\widehat\theta_t+Z_t$ and
$a_t=a^\star(\widetilde\theta_t) = \widetilde\theta_t/\|\widetilde\theta_t\|$ almost surely from Assumption \ref{assumption:TS-perturbations}. Then, we have \footnote{This result is inspired by the proof of  \cite[Theorem 3]{log_high_probability_bound}, which uses a similar cancellation for a deterministic UCB-type policy.}
\begin{align}
\begin{aligned}
\E_{t-1}[a_ta_t^\top]
&=
\frac12\E_{t-1}\left[
    \frac{
        (\widehat\theta_t+Z_t)
        (\widehat\theta_t+Z_t)^\top
    }{
        \|\widehat\theta_t+Z_t\|^2
    }
    +
    \frac{
        (\widehat\theta_t-Z_t)
        (\widehat\theta_t-Z_t)^\top
    }{
        \|\widehat\theta_t-Z_t\|^2
    }
\right] \\
&\succeq
\frac12\E_{t-1}\left[
    \frac{
        (\widehat\theta_t+Z_t)(\widehat\theta_t+Z_t)^\top
        +
        (\widehat\theta_t-Z_t)(\widehat\theta_t-Z_t)^\top
    }{
        (\|\widehat\theta_t\|+\|Z_t\|)^2
    }
\right] \\
&=
\E_{t-1}\left[
    \frac{
        \widehat\theta_t\widehat\theta_t^\top+Z_tZ_t^\top
    }{
        (\|\widehat\theta_t\|+\|Z_t\|)^2
    }
\right] \\
&\succeq
\E_{t-1}\left[
    \frac{
        Z_tZ_t^\top
    }{
        (\|\widehat\theta_t\|+\|Z_t\|)^2
    }
\right],
\end{aligned}
\end{align}
where we have used the fact that
$\|\widehat\theta_t\pm Z_t\|\leq \|\widehat\theta_t\|+\|Z_t\|$ and that the distribution of $Z_t$ is symmetric about the origin. On the event $\{\|\eta_t\|\leq L_\eta\}$,
\begin{align}
    &\|Z_t\|
    \leq
    \frac{c_{\mathrm{TS}}}
    {\sqrt{\lambda_{\min,t-1}}}
    \|\eta_t\|
    \leq
    c_{\mathrm{TS}}L_\eta
    \\
    &   \implies
\E_{t-1}[a_ta_t^\top]
\succeq
\frac{1}{(\Gamma+ c_{\mathrm{TS}}L_\eta)^2}
\E_{t-1}\!\left[
    Z_tZ_t^\top
    \mathbbm 1_{\{\|\eta_t\|\leq L_\eta\}}
\right] \succeq
\frac{
    \kappa_\eta c_{\mathrm{TS}}^2
}{
    (\Gamma+ c_{\mathrm{TS}}L_\eta)^2
}
V_{t-1}^{-1},
\end{align}
where the last inequality follows from
Assumption~\ref{assumption:TS-perturbations}.
\end{proof}

We can now state the main result concerning Thompson sampling. We introduce an assumption on $\lambda_{\min,m_T}$ that we discuss at the end of this section.

\begin{theorem}
\label{thm:TS}
Suppose Assumptions~\ref{assumption_1} and
\ref{assumption:TS-perturbations} hold, and let $c_{\mathrm{TS}}>0$
be independent of $T$. If 
$\lambda_{\min, m_T}\geq(\log T)^2$, then 
Algorithm~\ref{alg:thompson_sampling} satisfies
Properties~\ref{prop_3} and~\ref{prop_2} for $T$ sufficiently large.
\end{theorem}

\begin{proof}
Define $\mathcal G_{t-1}$ as given in \eqref{g_tdefined_1}. 
 As examined in \eqref{g_tdefined_1}-\eqref{g_tdefined_3}, for every sufficiently large $T$,
$\probb{\mathcal G_{T-1}}
    \geq
    1-\frac{1}{8T}$.
The proof proceeds through three high-probability events. On the event $\mathcal G_{t-1}$, we first establish the predictable-design lower bound
$
\E_{t-1}[a_ta_t^\top]\succsim 
V_{t-1}^{-1}
$
and use it to verify the drift condition in Property~\ref{prop_3}. For Property~\ref{prop_2}, we first show that, on the event $\mathcal G_{T-1}$,

$$
\sum_{t=M_T+1}^T
\|a_t-a^\star(\widehat\theta_t)\|^2
\lesssim
\sum_{t=M_T+1}^T
\lambda_{\min,t-1}^{-1}\|\eta_t\|^2.
$$

Theorem~\ref{theorem_unified_inference_bound} and Property~\ref{prop_3} , then, implies that $\lambda_{\min,t-1}\gtrsim\sqrt{t}$ simultaneously for $t=M_T+1,\ldots,T$ on a high-probability event. We intersect this event with $\mathcal G_{T-1}$ and the event
$
\sum_{t=M_T+1}^T t^{-1/2}\|\eta_t\|^2
\lesssim
\sqrt{T}\log T.
$
On their intersection,
$
\sum_{t=M_T+1}^T
\|a_t-a^\star(\widehat\theta_t)\|^2
\lesssim
\sqrt{T}\log T.
$
A union bound shows that this intersection has probability at least $1-1/(2T)$, thereby establishing Property~\ref{prop_2}.

We first verify Property~\ref{prop_3}. Fix $t=m_T+1,\ldots,T$ and assume that we are operating on the event $\mathcal G_{t-1}$. Then, we have
$\|\widehat\theta_t\|\leq \|\theta^{\star}\| + \theta_{\min}/2 \leq \theta_{\max} + \theta_{\min}/2 \leq 2\theta_{\max}$.
Moreover,
$V_{t-1}\succeq V_{m_T}$, implying
$\lambda_{\min,t-1}\geq(\log T)^2 \gtrsim 1$. Hence, for $T \geq 3$,
Lemma \ref{lem:TS-predictable-design} gives
\begin{align}
\label{eq:TS-predictable-design}
    \E_{t-1}[a_ta_t^\top]
    \succeq
    qV_{t-1}^{-1},
    \qquad
    q:=
    \frac{\kappa_\eta c_{\mathrm{TS}}^2}
    {(2\theta_{\max}+c_{\mathrm{TS}}L_\eta)^2}.
\end{align}

We next verify \eqref{eq:prop3_drift}. From \eqref{eq:Lt.Lt1}, it follows that
\begin{align}
\E_{t-1}[\Lambda_t-\Lambda_{t-1}]
&\geq
\frac{\Lambda_{t-1}^2}{1+(\lambda^{\textrm{TS}})^{-1}}
\E_{t-1}\!\left[
    a_t^\top V_{t-1}^{-2}a_t
\right] \notag\\
&=
\frac{\Lambda_{t-1}^2}{1+(\lambda^{\textrm{TS}})^{-1}}
\operatorname{tr}\!\left(
    V_{t-1}^{-2}\E_{t-1}[a_ta_t^\top]
\right)\notag\\
&\geq
\frac{q\Lambda_{t-1}^2}{1+(\lambda^{\textrm{TS}})^{-1}}
\operatorname{tr}(V_{t-1}^{-3})
\label{trace_loewner}\\
&\geq
\frac{q}{n^2(1+(\lambda^{\textrm{TS}})^{-1})}
\Lambda_{t-1}^{-1},
\notag
\end{align}
where \eqref{trace_loewner} follows from $\operatorname{tr}(AB) \geq \operatorname{tr}(AC)$ for the matrices $A \succeq 0, B \succeq C$ together with Lemma \ref{lem:TS-predictable-design}. The last inequality follows from the power mean inequality applied to the eigenvalues of $V_{t-1}$, giving
$\operatorname{tr}(V_{t-1}^{-3})
\geq\operatorname{tr}(V_{t-1}^{-1})^3/n^2 = \Lambda_{t-1}^{-3}/n^2$.
Thus, \eqref{eq:prop3_drift} holds with
\begin{align}\label{ts_c4_def}
    C_4
    :=
    \frac{\kappa_\eta c_{\mathrm{TS}}^2}
    {n^2(1+(\lambda^{\textrm{TS}})^{-1})
    (2\theta_{\max}+c_{\mathrm{TS}}L_\eta)^2}.
\end{align}

We next verify Property~\ref{prop_2}. By
Theorem~\ref{theorem_unified_inference_bound}, with probability at
least $1-\frac{1}{4T}$, for all $t=M_T+1,\ldots,T$,
\begin{align}\label{TS_inference_growth_prop_2}
    \lambda_{\min,t}
    \geq
    \frac12\sqrt{C_4 t}.
\end{align}

Since $V_t\preceq V_{t-1}+I$, on the same event
$
    \lambda_{\min,t-1}
    \geq
    \lambda_{\min,t}-1
    \geq
    \frac14\sqrt{C_4 t}$
for every sufficiently large $T$. By union bound, the intersection of $\mathcal G_{T-1}$ and the event given in \eqref{TS_inference_growth_prop_2} occurs with probability at least $1 - 3/(8T)$. On this intersection, Lemma \ref{lemma:xy} and \eqref{TS_inference_growth_prop_2} imply 
\begin{align}
\begin{aligned}
\sum_{t=M_T+1}^T \|a_t-a^\star(\widehat\theta_t)\|^2
&\leq \sum_{t=M_T+1}^T
\frac{
    4\|\widetilde\theta_t-\widehat\theta_t\|^2
}{
    \|\widehat\theta_t\|^2
} \\
&= \sum_{t=M_T+1}^T
\frac{
    4c_{\mathrm{TS}}^2
    \|V_{t-1}^{-1/2}\eta_t\|^2
}{
    \|\widehat\theta_t\|^2
} \\
&\leq
\frac{16c_{\mathrm{TS}}^2}{\theta_{\min}^2}
\sum_{t=M_T+1}^T
\lambda_{\min,t-1}^{-1}
\|\eta_t\|^2
\qquad \text{(Lemma~\ref{lemma:xy})} \\
&\leq
\frac{64c_{\mathrm{TS}}^2}
{\theta_{\min}^2\sqrt{C_4}}
\underbrace{\sum_{t=M_T+1}^T
t^{-1/2}\|\eta_t\|^2}_{:= S_T}
\qquad \text{by~\eqref{TS_inference_growth_prop_2}.}
\label{subgaussian_proof}
\end{aligned}
\end{align}
Hence, to show that $\sum_{t=M_T+1}^T \|a_t-a^\star(\widehat\theta_t)\|^2 \lesssim \sqrt{T}\log T$ with probability at least $1-\frac{1}{2T}$, as required by Property \ref{prop_2}, it remains to establish that $S_T \lesssim \sqrt{T}\log T$ with probability at least $1-\frac{1}{8T}$. Intersecting this event with $\mathcal G_{T-1}$ and the event in \eqref{TS_inference_growth_prop_2}, and applying a union bound, yields the desired result.

To this end, first, we will show that $\mathbb E[S_T] \lesssim \sqrt T$ and apply a concentration inequality to bound $S_T$ with high probability. By Assumption \ref{assumption:TS-perturbations}, for every $u\in\mathbb S^{n-1}$,
$
\mathbb E(u^\top\eta_t)^2\leq 1
$
and
$
\mathbb E(u^\top\eta_t)^4\leq K_\eta.
$
Denote the $i^{\text{th}}$ entry of $\eta_t$ as $\eta_t^i$. Let $e_1, \ldots, e_n$ denote the $n$ canonical basis vectors in $\mathbb{R}^n$. Taking $u=e_i$ gives
$
\mathbb E[(\eta_t^i)^2]\leq 1
$
and
$
\mathbb E[(\eta_t^i)^4]\leq K_\eta
$
for every $i=1,\ldots,n$. Hence,
\begin{align}
\mathbb E\|\eta_t\|^2
&=
\sum_{i=1}^n \mathbb E[(\eta_t^i)^2]
\leq n.
\end{align}
It follows that
\begin{align}
\mathbb E[S_T]
&\leq
n\sum_{t=M_T+1}^T t^{-1/2}
\leq
2n\sqrt{T}.
\end{align}
Next, we apply the Chebyshev's inequality,
\begin{align}
\mathbb P\left(
S_T-\mathbb E[S_T]
\geq
\sqrt{8T\operatorname{var}[S_T]}
\right)
\leq
\frac{1}{8T}.
\end{align}
Suppose that $\operatorname{var}[S_T] \leq n^2 K_\eta(1 + \log T)$. Then, with probability at least $1-\frac{1}{8T}$,
\begin{align}
S_T
&\leq
2n\sqrt{T}
+
2\sqrt{2}\,n \sqrt{K_\eta} \sqrt{T(1+\log T)}
\lesssim
\sqrt{T}\log T.
\end{align}
which combined with \eqref{subgaussian_proof} and a union bound proves $\sum_{t=M_T+1}^T \|a_t-a^\star(\widehat\theta_t)\|^2 \lesssim \sqrt{T}\log T$ with probability at least $1-1/(2T)$, as required by Property \ref{prop_2} which holds with
$
C_2
:=
\frac{256 n (1 + \sqrt{K_\eta}) c_{\mathrm{TS}}^2}
{\theta_{\min}^2\sqrt{C_4}}
$ for a large enough $T$.

To conclude, we show that $\operatorname{var}[S_T] \leq n^2 K_\eta(1 + \log T)$. Since the perturbations $(\eta_t)_{t\geq1}$ are independent,
\begin{align}
\begin{aligned}\label{var_bound_ts_1}
\operatorname{var}[S_T]
=
\sum_{t=M_T+1}^T
t^{-1}\operatorname{var}[\|\eta_t\|^2] \leq
\sum_{t=M_T+1}^T
t^{-1}\mathbb E\|\eta_t\|^4. 
\end{aligned}
\end{align}
To bound the sum, notice that, 
\begin{align}
\begin{aligned}
\mathbb E\|\eta_t\|^4
=\mathbb E \left [ \left( \sum_{i = 1}^n (\eta_t^i)^2\right)^2\right]  = 
\sum_{i,j=1}^n
\mathbb E\left[(\eta_t^i)^2(\eta_t^j)^2\right] \leq
\sum_{i,j=1}^n
\sqrt{
\mathbb E[(\eta_t^i)^4]
\mathbb E[(\eta_t^j)^4]
}
\leq n^2 K_\eta,
\end{aligned}
\end{align}
where the first inequality follows from the Cauchy--Schwarz inequality. Applying this bound to \eqref{var_bound_ts_1}, and using
$
n^2K_\eta\sum_{t=1}^T t^{-1}
\leq
n^2K_\eta(1+\log T)$ finalizes the proof.
\end{proof}

Algorithm~\ref{alg:thompson_sampling} satisfies the inference and
action-concentration conditions required by
Theorem~\ref{theorem_unified_regret_bound}, and hence achieves the
order-optimal regret rate.
We  end this section by examining the $\lambda_{\min, m_T} \geq \log^2 T$ assumption. While this may seem restrictive, a simple exploratory policy over this burn-in period up until $t=m_T$ produces such a growth. To illustrate, let $e_1, \ldots, e_n$ denote the $n$ canonical basis vectors in $\mathbb{R}^n$. Playing $e_1, \ldots, e_n$ as the actions in sequence, one would obtain $V_{kn} = \lambda I + k \sum_{i=1}^n e_i e_i^\top = (\lambda + k) I$ for which $\lambda_{\min, kn} = \lambda + k$. With this exploratory policy for $kn \approx m_T \approx \frac{1}{3}\sqrt{T}\log{T}$, we will have
\begin{align}
    \lambda_{\min, m_T} \approx \lambda + \frac{1}{3n} \sqrt{T}\log{T} \gtrsim \log^2(T)\label{forced_exp_eq}
\end{align}
for a large enough $T$, as required. Running exploratory burn-ins in adaptive control literature are common, e.g., see \cite{etc_paper,bandit_algorithms_book}.

\section{Analysis of UCB}\label{sec:ucb}

The UCB algorithm selects actions by being optimistic within a confidence
set centered at $\widehat\theta_t$. At round $t$, it maintains the
ellipsoidal confidence set
\begin{align}
\mathcal E_t
:=
\left\{
\theta\in\mathbb R^n:
\|\theta-\widehat\theta_t\|_{V_{t-1}}
\le \rho_T
\right\},
\end{align}
where $\rho_T\simeq\sqrt{\log T}$, as in \eqref{eq:rhoT.def}. With high probability, this set contains
the true parameter $\theta^\star$ simultaneously over all  rounds; 
see \cite{improved_algorithms_for_stochastic_bandits}. The UCB
action is chosen as
\begin{align}
\theta_t'
\in
\argmax_{\theta\in\mathcal E_t}
\theta^\top a^\star(\theta)\quad  a_t =
a^\star(\theta_t'),
\end{align}
i.e., the optimal action corresponding to the most favorable or optimistic parameter estimate within the confidence ellipsoid $\mathcal{E}_t$. In what follows, we consider an optimism-driven algorithm that again maximizes the same objective, but swaps 
the ellipsoid $\mathcal E_t$ 
with a \emph{box}-shaped confidence set $\mathcal{B}_t$, described below. Consider a spectral decomposition of $V_{t-1}$,
\begin{align}
V_{t-1}
=
\sum_{i=1}^n
\lambda_{i,t-1}
v_{i,t-1}v_{i,t-1}^\top.
\label{eq:spectral_decomp}
\end{align}
with $v_{i, t-1}$ being orthonormal vectors that span $\mathbb{R}^n$. Define
\begin{align}\label{eq:define_bt}
\mathcal B_t
=
\left\{
\theta\in\mathbb R^n:
\left|
v_{i,t-1}^\top(\theta-\widehat\theta_t)
\right|
\le
\frac{\rho_T}{\sqrt{\lambda_{i,t-1}}},
\quad i=1, \ldots, n
\right\}.
\end{align}
The set $\mathcal{B}_t$ is a rectangle centered at $\widehat\theta_t$ with sides along the eigenvectors of $V_{t-1}$. 
Since
\begin{align}
\lambda_{i,t-1}
\bigl(v_{i,t-1}^\top(\theta-\widehat\theta_t)\bigr)^2
\leq
\|\theta-\widehat\theta_t\|_{V_{t-1}}^2, \quad \text{for } i=1, \ldots, n,
\end{align}
we have $\mathcal E_t\subseteq\mathcal B_t$, i.e., our considered confidence set always contains the one in standard UCB implementations.
The modified optimistic parameter and its corresponding optimal action are then given by
\begin{align}
\widetilde\theta_t
\in
\argmax_{\theta\in\mathcal B_t} \ 
\theta^\top a^\star(\theta),\qquad
a_t
=
a^\star(\widetilde\theta_t).
\end{align}
The procedure is formally presented in Algorithm \ref{alg:ucb}. The following theorem establish Properties \ref{prop_3} and \ref{prop_2} for Algorithm \ref{alg:ucb}. 

\begin{algorithm}
\caption{UCB Algorithm, adapted from \cite{improved_algorithms_for_stochastic_bandits, context_bandits_ucb}}
\label{alg:ucb}
\begin{algorithmic}[1]
\STATE \textbf{Data:} $\widehat{\theta}_1$, T, $\lambda^{\textrm{UCB}} > 0$
\STATE  $V_0\gets\lambda^{\textrm{UCB}} I_n$  
\STATE  $\rho_T \gets  M\sqrt{
    n\log\left(4T (1+T/\lambda^{\textrm{UCB}})
    \right)
}
+
\sqrt{\lambda^{\textrm{UCB}}}\,\theta_{\max}$
\FOR{$t = 1, \dots, T$}
    \STATE $\tilde \theta_t
\in
\argmax_{\theta\in\mathcal B_t}
\theta^\top a^\star(\theta)$
    \STATE $a_t\gets\arg \max_{a \in \Acal} \, a^\top \tilde \theta_t$
    \STATE Observe  $Y_t$ 
    \STATE  $V_t\gets V_{t-1} + a_t a_t^\top$
    \STATE  $\widehat{\theta}_{t + 1}\gets V_{t}^{-1} \sum_{s = 1}^t a_s Y_{s}$
\ENDFOR
\end{algorithmic}
\end{algorithm}

\begin{theorem}
\label{thm:UCB}
Suppose Assumption~\ref{assumption_1} holds. If
$
\lambda_{\min, m_T}\geq(\log T)^2
$, then  Algorithm~\ref{alg:ucb} satisfies
Properties~\ref{prop_3} and~\ref{prop_2} for $T$ sufficiently large.
\end{theorem}

\begin{proof}
Define $\mathcal G_{t-1}$ as in \eqref{g_tdefined_1}. Since
$
\lambda_{\min, m_T}\geq(\log T)^2,
$
\eqref{g_tdefined_3} follows
for every sufficiently large $T$. Fix a $t$ from $ m_T + 1 \dots T$ and work on $\mathcal G_{t-1}$. From \eqref{eq:Lt.Lt1},
\begin{align}
\begin{aligned}
\Lambda_t-\Lambda_{t-1}
    &\geq \frac{\Lambda_{t-1}^2}{1+(\lambda^\mathrm{UCB})^{-1}} a_t^\top V_{t-1}^{-2}a_t.
    \end{aligned}
    \label{eq:Lt.Lt1_ucb}
\end{align}
To verify Property~\ref{prop_3}, it suffices to show that $a_t^\top V_{t-1}^{-2}a_t \gtrsim \Lambda_{t-1}^{-3}$. From the spectral decomposition of $V_{t-1}$ in \eqref{eq:spectral_decomp}, we can obtain
\begin{align}
a_t^\top V_{t-1}^{-2}a_t
\geq
\frac{(v_{\min,t-1}^\top a_t)^2}{\lambda_{\min,t-1}^2}.
\label{eq:sherm.4.num}
\end{align}
Suppose that 
\begin{align} \label{new_addition_ucb}
(v_{\min,t-1}^\top a_t)^2
\geq \frac{\lambda^{\mathrm{UCB}}}
{(2+\sqrt n)^2}
\frac{1}{\lambda_{\min,t-1}}.
\end{align}
Then, \eqref{eq:sherm.4.num} implies 
\begin{align}
a_t^\top V_{t-1}^{-2}a_t
\geq
\frac{\lambda^{\mathrm{UCB}}}
{(2+\sqrt n)^2\lambda_{\min,t-1}^3}.
\label{eq:sherm.2.num}
\end{align}
Since
$
\operatorname{tr}(V_{t-1}^{-1})
\leq
\frac{n}{\lambda_{\min,t-1}},
$
we have $\frac{1}{\lambda_{\min,t-1}^3}
\geq
\frac{
    \operatorname{tr}(V_{t-1}^{-1})^3
}{
    n^3
}
=
\frac{1}{n^3\Lambda_{t-1}^3}$. Plugging this bound in the above inequality, we get
\begin{align}\label{a_t_lower_ucb}
a_t^\top V_{t-1}^{-2}a_t
\geq
\frac{(v_{\min,t-1}^\top a_t)^2}{\lambda_{\min,t-1}^2}
\geq
\frac{\lambda^{\mathrm{UCB}}}
{(2+\sqrt n)^2 n^3 \Lambda_{t-1}^3}.
\end{align}
Combining \eqref{eq:Lt.Lt1_ucb} and \eqref{a_t_lower_ucb} implies $\Lambda_t - \Lambda_{t-1} \gtrsim \Lambda_{t-1}^{-1}$  resulting in \eqref{eq:prop3_drift} with $ C_4^{-1}
:= {(2+\sqrt n)^2
    (1+ 1/\lambda^{\mathrm{UCB}})
n^3}/{\lambda^{\mathrm{UCB}}}$. Hence, to finalize the proof of Property~\ref{prop_3}, it remains to show \eqref{new_addition_ucb}. 
To that end, notice that maximizing the UCB objective over $\mathcal B_t$ with a spherical action set amounts to
maximizing $\|\theta\|$ over $\mathcal B_t$. Thus, a closed form solution can be given as
\[
\widetilde\theta_t
=
\widehat\theta_t
+
\rho_T
\sum_{i=1}^n
\frac{s_{i,t}}{\sqrt{\lambda_{i,t-1}}}
v_{i,t-1},
\qquad
s_{i,t}
\in
\operatorname{sign}
\bigl(v_{i,t-1}^\top\widehat\theta_t\bigr),
\]
with the sign at zero chosen according to a fixed deterministic
tie-breaking rule. Since
\begin{align}
\label{min_eiq_ucb}    
\left|v_{\min,t-1}^\top\widetilde\theta_t\right|
=
\left|v_{\min,t-1}^\top\widehat\theta_t\right|
+
\frac{\rho_T}{\sqrt{\lambda_{\min,t-1}}}
\geq
\frac{\rho_T}{\sqrt{\lambda_{\min,t-1}}}
>0,
\end{align}
we have $\widetilde\theta_t\neq0$. Thus, 
$
a_t
=
\frac{\widetilde\theta_t}{\|\widetilde\theta_t\|}.
$
Also,
\[
\begin{aligned}
\|\widetilde\theta_t\|
\leq
\|\widehat\theta_t\|
+
\rho_T
\left(
\sum_{i=1}^n \lambda_{i,t-1}^{-1}
\right)^{1/2} 
\leq
2\theta_{\max}
+
\rho_T\sqrt{\frac{n}{\lambda_{\min,t-1}}} 
\leq
2\theta_{\max}
+
\rho_T\sqrt{\frac{n}{\lambda^{\mathrm{UCB}}}}.
\end{aligned}
\]
where we used $\| \widehat \theta_t \| \leq 2 \theta_{\max}$ on $\mathcal{G}_{t-1}$. Therefore, \eqref{min_eiq_ucb} yields
\begin{align}
\begin{aligned}
(v_{\min,t-1}^\top a_t)^2
=
\frac{(v_{\min,t-1}^\top\widetilde\theta_t)^2}
{\|\widetilde\theta_t\|^2} &\geq
\frac{\rho_T^2}
{\lambda_{\min,t-1}
\bigl(
2\theta_{\max}
+
\rho_T\sqrt{n/\lambda^{\mathrm{UCB}}}
\bigr)^2} \\
& \geq \frac{\lambda^{\mathrm{UCB}}}
{(2+\sqrt n)^2}
\frac{1}{\lambda_{\min,t-1}}.
\end{aligned}\label{rho_t_lower_bounded_constant}
\end{align}
where the last line comes from the definition of $\rho_T$ in \eqref{eq:rhoT.def}, which implies that
$
\rho_T
\geq
\sqrt{\lambda^{\mathrm{UCB}}}\,\theta_{\max}
$
and therefore
\[
\begin{aligned}
2\theta_{\max}
+
\rho_T\sqrt{\frac{n}{\lambda^{\mathrm{UCB}}}}
\leq
\frac{2\rho_T}{\sqrt{\lambda^{\mathrm{UCB}}}}
+
\frac{\rho_T\sqrt n}{\sqrt{\lambda^{\mathrm{UCB}}}} =
\frac{\rho_T}{\sqrt{\lambda^{\mathrm{UCB}}}}
(2+\sqrt n).
\end{aligned}
\]
This proves \eqref{new_addition_ucb} and finalizes the proof of Property \ref{prop_3}.

We next verify Property~\ref{prop_2}. By
Theorem~\ref{theorem_unified_inference_bound}, with probability at
least $1-\frac{1}{4T}$,
\begin{align} \label{even_1_ucb}
\lambda_{\min,t}
\geq
\frac12\sqrt{C_4 t},
\qquad
t=M_T+1,\ldots,T.
\end{align}

Since $V_t\preceq V_{t-1}+I$, on the same event,
$
\lambda_{\min,t-1}
\geq
\lambda_{\min,t}-1
\geq
\frac14\sqrt{C_4 t}
$
for every sufficiently large $T$. Moreover, by the definition of $\mathcal B_t$ in \eqref{eq:define_bt},
\begin{align}
\begin{aligned}
\|\widetilde\theta_t-\widehat\theta_t\|^2
&\leq
\rho_T^2 \sum_{i=1}^n \lambda_{i,t-1}^{-1}
=
\rho_T^2\operatorname{tr}(V_{t-1}^{-1}) \leq
\frac{n\rho_T^2}{\lambda_{\min,t-1}}.
\end{aligned}\label{rho_t_used_prop_2}
\end{align}
Since $\rho_T^2\leq c_\rho\log T$ for some $T$-independent constant
$c_\rho>0$, on the intersection of $\mathcal G_{T-1}$ and the high-probability event in \eqref{even_1_ucb}, we obtain
\begin{align}
\|a_t-a^\star(\widehat\theta_t)\|^2
\leq
\frac{16}{\theta_{\min}^2}
\|\widetilde\theta_t-\widehat\theta_t\|^2 \leq 
\frac{64n c_\rho}
{\theta_{\min}^2\sqrt{C_4}}
\frac{\log T}{\sqrt t},
\end{align}
where the first inequality follows from Lemma \ref{lemma:xy} applied on $\mathcal G_{T-1}$ and the second one comes from \eqref{even_1_ucb} and  \eqref{rho_t_used_prop_2}. Therefore,
\begin{align}
\begin{aligned}
\sum_{t=M_T+1}^T
\|a_t-a^\star(\widehat\theta_t)\|^2
&\leq
\frac{64n c_\rho}
{\theta_{\min}^2\sqrt{C_4}}
\log T
\sum_{t=M_T+1}^T t^{-1/2} \\
&\leq
\frac{128n c_\rho}
{\theta_{\min}^2\sqrt{C_4}}
\sqrt T\log T.\label{final_result_ucb}
\end{aligned}
\end{align}
Finally, by \eqref{g_tdefined_3},
and the union bound, the intersection of $\mathcal G_{T-1}$ and the event in \eqref{even_1_ucb} occur with probability at least
$
1-\frac{1}{8T}-\frac{1}{4T}
\geq
1-\frac{1}{2T}$.
Thus, Property~\ref{prop_2} holds with
$C_2
:=
\frac{128n c_\rho}
{\theta_{\min}^2\sqrt{C_4}}$, as shown in \eqref{final_result_ucb}. This completes the proof.
\end{proof}

We emphasize that the proof of Theorem~\ref{thm:UCB} does not rely on the inclusion
$
\mathcal E_t\subseteq \mathcal B_t,
$
nor, more generally, on the optimism property of the construction. The enlargement from the ellipsoidal confidence set \(\mathcal E_t\) to the box \(\mathcal B_t\) was introduced for two reasons. First, it preserves the optimism-in-the-face-of-uncertainty principle of UCB, given that $\theta^\star\in\mathcal E_t\subseteq\mathcal B_t$ with high probability, and second, the optimization of the UCB objective over a box 
admits an explicit characterization. That latter aids mathematical analysis of Property~\ref{prop_3}. It is of import to notice that optimism as a concept never featured in the proof of Theorem~\ref{thm:UCB}. In particular, the argument never invokes \(\theta^\star\in\mathcal B_t\), the inclusion \(\mathcal E_t\subseteq\mathcal B_t\), or a comparison between the optimistic value attained by \(\widetilde\theta_t\) and the value associated with \(\theta^\star\). The proof, however, requires instead that the selected action generates sufficient information growth through the self-correction condition, while remaining sufficiently concentrated around the greedy action. Said differently, \emph{optimism is therefore a means to the end of generating the desired exploration behavior, rather than a prerequisite for the analysis.}

This observation allows the radius of \(\mathcal B_t\) to be chosen independently of the statistical confidence radius $\rho_T$. In particular, consider replacing \(\rho_T\) in the definition of \(\mathcal B_t\) by a fixed constant
$$
\rho\geq
\sqrt{\lambda^{\mathrm{UCB}}}\,\theta_{\max}.
$$
The resulting set need not contain \(\mathcal E_t\), and hence it need not be a confidence set for \(\theta^\star\). Nevertheless, the self-correction argument remains unchanged. Indeed, in the proof of Property \ref{prop_3}, we only require a lower bound on $\rho_T$ since at \eqref{rho_t_lower_bounded_constant} we replace $\rho_T$ with $
\sqrt{\lambda^{\mathrm{UCB}}}\,\theta_{\max},
$
a \(T\)-independent drift constant. This yields
$
\Lambda_t-\Lambda_{t-1}
\gtrsim
\Lambda_{t-1}^{-1}$ on $\mathcal G_{t-1}$.  Consequently, the same high-probability information-growth guarantee
$
\Lambda_t\gtrsim\sqrt t
$ continues to hold under the initialization condition-- $\lambda_{\min, m_T}\geq (\log T)^2$, whenever $\rho\geq
\sqrt{\lambda^{\mathrm{UCB}}}\,\theta_{\max}.$
The argument establishing Property~\ref{prop_2} is likewise preserved. Replacing \(\rho_T\) by \(\rho\) in \eqref{rho_t_used_prop_2} gives
$
\|\widetilde\theta_t-\widehat\theta_t\|^2
\leq
n\rho^2 \lambda_{\min,t-1}^{-1}.
$ Since self-correction yields \(\lambda_{\min,t-1}\gtrsim\sqrt t\), the resulting cumulative deviation satisfies
\begin{align}
\sum_{t=M_T+1}^T
\|a_t-a^\star(\widehat\theta_t)\|^2
\lesssim
\sqrt T,
\end{align}
and hence, in particular, Property~\ref{prop_2}. This illustrates the flexibility of the self-correction viewpoint: \emph{essential for optimal regret is sufficient inference growth together with controlled deviation from the estimated greedy action, not optimism itself}.

We formalize the observation about information growth with UCB over a box set with a constant $\rho$ in the following result. This variant will prove useful in understanding the behavior of the various algorithms over finite horizons in the next section.
\begin{corollary}[Constant-radius UCB]
\label{cor:constant_radius_ucb}
Suppose Assumption~\ref{assumption_1} holds.
Consider Algorithm~\ref{alg:ucb} with $\rho_T$ in the definition of $\mathcal B_t$ replaced with a constant
$
\rho\geq
\sqrt{\lambda^{\mathrm{UCB}}}\,\theta_{\max}
$. If 
$\lambda_{\min, m_T}\geq(\log T)^2$, then the policy generated by this algorithm satisfies Properties~\ref{prop_3} and~\ref{prop_2} for large $T$.
\end{corollary}

{ We end this section by noting a related information-growth mechanism for
a substantially modified UCB algorithm studied in
\cite{log_high_probability_bound}. Their LinUCB-VN algorithm operates in batches of $2(n-1)$ rounds. At each batch $t$, it constructs pairs of actions by perturbing the normalized least-squares estimate in the positive and negative directions of  $n-1$ eigenvectors of $V_{t-1}$:
\[
\frac{\widehat \theta_t}{\|\widehat \theta_t\|}
\pm
\frac{1}{\sqrt{\lambda_{\min,t-1}}}
v_{i}(V_{t-1}),
\qquad
i=1,\ldots,n-1,
\]
and then plays the normalized version of these vectors. Each batch collects two observations for every one of these $n-1$ directions. This produces an information accumulation guarantee of the form $\lambda_{\min,t} \gtrsim \sqrt{\lambda_{\max}(V_t)}$, with $\lambda_{\max}(V_t)$ denoting the maximum eigenvalue of $V_t$. The exploration dynamics are closely related to the
self-correcting mechanism formalized in Property~\ref{prop_3}. With LinUCB-VN, almost all directions, including the most poorly explored direction, get excited via perturbations that scale with $\lambda_{\min,t}^{-1/2}$.

There are, however, important differences between \cite{log_high_probability_bound} and our analysis. The mechanism in \cite{log_high_probability_bound} is enforced pathwise through
$2(n-1)$ carefully paired actions at every batch and is tailored to a substantially modified UCB procedure. Such a pathwise construction does
not directly describe randomized policies such as Thompson sampling, nor does it apply immediately to simpler one-action-per-round UCB variants
such as Algorithm~\ref{alg:ucb}. Our analysis, on the other hand,  abstracts the self-correction
phenomenon to a conditional per-step drift requirement. This weaker formulation is sufficient to obtain the $\sqrt{t}$-order $\lambda_{\min,t}$ growth needed in our regret analysis while allowing the underlying exploration mechanism to be either deterministic or randomized.
}

\section{What Happens in Short Horizons?}\label{sec:numerics}

In this section, we examine the finite-horizon limitations of our theoretical guarantees. Our analyses of TS and UCB require the horizon $T$ to be sufficiently large. In particular, the initialization condition $\lambda_{\min,m_T}\geq(\log T)^2$ ensures, through \eqref{eq:self_normalized_AY}, that the good-history event satisfies
\eqref{g_tdefined_3} for sufficiently large $T$. The forced-exploration phase given in \eqref{forced_exp_eq} can achieve this initialization. With this policy, cycling through an orthonormal basis of $\mathbb R^n$ yields $\lambda_{\min,t} \geq \lambda+\lfloor t/n\rfloor.
$ The minimum eigenvalue, therefore, grows at the rate of $t/n$ which is the fastest allowable growth rate since for any policy,
\begin{align}
\lambda_{\min,t}
\leq
\frac{\operatorname{tr}(V_t)}{n}
\leq 
\lambda+\frac{\sum_{s = 1}^t \|a_s\|^2}{n} \leq \lambda+\frac{t}{n} 
\end{align}
Consequently, even with the fastest allowable growth rate of $\lambda_{\min,t}$, the initialization cannot be achieved whenever
\begin{align}\label{impossible}
\lambda+\frac{m_T}{n}
<
(\log T)^2.
\end{align}
Substituting
$
m_T=\lceil \frac{1}{3}\sqrt{T}\log T\rceil
$
and ignoring regularization ($\lambda$) and rounding terms, the initialization condition requires
$
T
\gtrsim
9n^2(\log T)^2.
$ This highlights the finite-horizon limitation of the initialization scheme.

This raises the question of how the algorithms presented in this paper behave at more moderate horizons, where the initialization condition may not be attainable within \(m_T\). We therefore next investigate the finite-horizon performance of TS and UCB in regimes where the assumptions underlying our asymptotic guarantees need not yet hold.

We run our experiments on the Euclidean unit-ball action set in $\mathbb R^{10}$; all selected actions lie on its boundary. For all three algorithms, we run the experiment $10$ times with different $\theta^\star$ initializations and noise sequences across repetitions for $T = 10^5$. For a paired comparison, all three policies use the same true parameter and observation-noise sequence within each run. In each of the ten independent repetitions, we draw $g\sim\mathcal N(0,I_{10})$ and set $\theta^\star=g/\|g\|_2$. Rewards are $Y_t=a_t^\top\theta^\star+\varepsilon_t$, with independent $\varepsilon_t\sim\mathcal N(0,0.1^2)$, sampled before each repetition starts so that all three algorithms share the same noise sequence within that repetition. We set $\lambda_{\rm TS}=\lambda_{\rm UCB}=c_{\rm TS}=1$, initialize $V_0=I_{10}$ and $\widehat{\theta}_1=0$. These selections allow us to examine the $T$ values where the initialization requirement cannot be reached even with the $t/n$ order growth of $\lambda_{\min,t}$. Indeed, with $\lambda = 1, n = 10, T = 10^5$,
\begin{align}
\lambda+\frac{m_T}{n}
=
1+\frac{1214}{10}
<
132.55
\simeq
(\log T)^2,
\end{align}
Thus, the impossibility of satisfying the initialization condition follows from \eqref{impossible}. In the preceding, we compare three policies: linear Thompson sampling, given as Algorithm \ref{alg:thompson_sampling} with $c_{\rm TS}=1$ and $\eta_t \sim \mathcal N(0, I_{10})$, modified UCB given as Algorithm \ref{alg:ucb} with $\rho_T$ defined in \eqref{eq:rhoT.def}, and the same modified UCB policy with $\rho_T$ replaced with a constant $\rho$, as examined in Corollary \ref{cor:constant_radius_ucb}. In this section, we denote the modified UCB algorithms as Box UCB. For these experiments, the radius $\rho_T$ is calculated as
\[
\rho_T=0.1\sqrt{10\log\!\bigl(4 \times 10^5(1+10^5)\bigr)}+1 \simeq 2.56
\]
and the experiments are conducted with \(\rho=1\). Our theoretical results apply to all three algorithms, establishing high-probability inference guarantees of the form
$\Lambda_t \gtrsim \sqrt{t},
$ together with \(\mathcal O(\sqrt{T}\log T)\) regret guarantees, but only after the initialization requirement is established which cannot happen in these experiments at $T = 10^5$ for $n = 10$.

Figure~\ref{fig:l2_information} reports $\log\lambda_{\min,t}/\log t$ for $t\ge2$, and Figure~\ref{fig:l2_regret} reports $\log\mathscr{R}_{\theta^\star}(t)/\log t$. For each quantity, we first average the underlying curves across repetitions and then plot $\log \overline{X}_t/\log t$, where
$
\overline{X}_t:=\frac{1}{10}\sum_{i=1}^{10}X_t^i,
$
and $X_t^i$ denotes either $\lambda_{\min,t}$ or $\mathscr{R}_{\theta^\star}(t)$ in repetition $i$, with $i=1,\ldots,10$. Notice that if $X_t \simeq Ct^\alpha$, then
$
\log X_t/\log t
\simeq 
\alpha+\log C/ \log t,
$
so multiplicative constants generate slowly vanishing finite-horizon corrections. Thus, both $\log \lambda_{\min,t}/\log t$ and $\log \mathscr{R}_{\theta^\star}(t)/\log t$ retain finite-time effects from multiplicative constants. Because of this, we also estimate their growth exponents using the last $T/100$ time instances. For each algorithm, we first average $\lambda_{\min,t}$ and $\mathscr{R}_{\theta^\star}(t)$ across the $10$ repetitions and then, over the final $1\%$ of rounds, fit
$
\log \overline{X}_t = a\log t+b,
$
with $\overline{X}_t$ equal to the corresponding averaged curve. The fitted slope $a$ is reported as the estimated polynomial growth exponent for the regret or $\lambda_{\min,t}$ curves of the algorithms.

Figures~\ref{fig:l2_information} and~\ref{fig:l2_regret} show that Thompson sampling and Box UCB with $\rho=1$ exhibit very similar behavior in both information growth and cumulative regret. By contrast, Box UCB with the horizon-dependent radius $\rho_T$ produces systematically larger values of $\lambda_{\min,t}$, reflecting more aggressive exploration, and correspondingly larger cumulative regret. With $\rho_T$, the box UCB generates $\log\lambda_{\min,t}/\log t$ close to $0.6$, while $\log\mathscr{R}_{\theta^\star}(t)/\log t$ is also substantially larger than for the other two policies. These raw ratios, however, should not be interpreted directly as polynomial growth exponents. The tail regressions give a cleaner measure of the polynomial growth rates. Fitting $\log \overline{X}_t=a\log t+b$ to the averaged curves over the final $1\%$ of rounds yields $\lambda_{\min,t}$ exponents of $0.5091$, $0.5162$, and $0.5367$ for TS, Box UCB with $\rho=1$, and Box UCB with $\rho_T=2.5624$, respectively. The corresponding regret exponents are $0.4986$, $0.5030$, and $0.5167$. Thus, for TS and Box UCB with fixed radius, both information growth and regret are very close to $t^{1/2}$ behavior despite the visibly larger values of the raw log-ratios. The horizon-dependent UCB rule retains a stronger finite-horizon deviation away from $1/2$, with the deviation more pronounced for $\lambda_{\min,t}$. Nevertheless, the fitted slopes are substantially closer to $1/2$ than the raw log-ratios suggest, supporting the interpretation that much of the apparent excess growth of $\lambda_{\min,t}$ at $T=10^5$ in the raw log-ratios is a finite-horizon transient rather than evidence of a fundamentally different polynomial order. This phenomenon is consistent with previous empirical observations favoring randomized methods over optimism-based algorithms. As noted in \cite{abeille2025when},  
\begin{quote}Despite the strong theoretical performance of optimistic algorithms, randomised algorithms, such as Thompson sampling, have been shown to perform better in practice.
\end{quote}

To further investigate this behavior, we examine the self-correction mechanism underlying information acquisition. Recall that Property~\ref{prop_3} concerns the conditional drift of $
\Lambda_t$. Guided by this, we estimate:
\[
\frac{1}{T} \sum_{t = 1}^T \mathbb E_{t-1}\left[
\Lambda_{t-1}(\Lambda_t-\Lambda_{t-1})\right] 
\]
where for $t \geq m_T+1$, the summand $\mathbb E_{t-1}\left[
\Lambda_{t-1}(\Lambda_t-\Lambda_{t-1})\right] $ is bounded below by $C_4$ on the good-history event in Property \ref{prop_3}. Writing $N=10$ for the number of repetitions, we define the estimate
\begin{align}
\widehat C_4^{\mathrm{all}}
:=
\frac{1}{NT}
\sum_{r=1}^{N}\sum_{t=1}^{T}
\Lambda_{t-1}^{(r)}
\left(\Lambda_t^{(r)}-\Lambda_{t-1}^{(r)}\right),
\label{eq:c4_empirical_all}
\end{align}
where $r$ indexes repetitions and $\Lambda_0^{(r)}=\lambda/n$. To examine the behavior near the end of the experiment, we also compute the tail estimate using only the final $T/100$ observations from each of the ten repetitions $
\widehat C_4^{\mathrm{tail}}
:=
\frac{100}{NT}
\sum_{r=1}^{N}\sum_{t=0.99T+1}^{T}
\Lambda_{t-1}^{(r)}
\left(\Lambda_t^{(r)}-\Lambda_{t-1}^{(r)}\right).$

Nevertheless, the empirical averages and the theoretical constant $C_4$ have different interpretations: the former summarize drift along the observed trajectories, whereas the latter provides a uniform conditional lower bound under the assumptions of the analysis. Consequently, these estimates diagnose the self-correction mechanism but do not establish Property~\ref{prop_3} at the simulated horizon. Comparing the full-horizon and tail estimates also reveals whether the average scaled drift changes substantially near the end of the experiment; close agreement indicates that the measured self-correction is not driven primarily by the initial rounds.

The full-horizon estimates $\widehat C_4^{\mathrm{all}}$ are $0.011808$, $0.011770$, and $0.070798$ for TS, Box UCB with $\rho=1$, and Box UCB with $\rho_T$, respectively. The corresponding tail estimates $\widehat C_4^{\mathrm{tail}}$ are $0.012114$, $0.012072$, and $0.075558$. Thus, the full-horizon and tail estimates differ by approximately $2.6\%$, $2.6\%$, and $6.7\%$, respectively, indicating that the average scaled information drift changes only modestly in the final portion of the experiment, especially for TS and box UCB with $\rho = 1$. These two algorithms exhibit nearly identical empirical drift, whereas the horizon-dependent Box UCB produces estimates approximately six times as large. This stronger information acquisition accompanies the larger cumulative regret observed for the horizon-dependent policy, illustrating the cost of its more aggressive exploration.

The experiments therefore expose a finite-horizon cost of theoretically order-optimal optimism: the horizon-dependent UCB radius increases both information acquisition and cumulative regret at practical horizons. Furthermore, comparing $\rho$ vs $\rho_T$ for the box UCB, illustrates how an optimism-agnostic viewpoint can serve not only as an analysis framework, but also as a design principle. By identifying sufficient inference growth and action concentration as the key requirements, our framework allows exploration mechanisms to be evaluated and tuned without requiring optimism. Corollary~\ref{cor:constant_radius_ucb} demonstrates this flexibility: under its assumptions, constant-radius Box UCB satisfies the conditions needed for order-optimal regret, up to logarithmic factors. In our experimental setting, $\rho=1$ meets the corollary’s radius requirement and even though the algorithm fails the $\lambda_{\min, m_T} \geq (\log T)^2$ requirement, it achieves empirical performance comparable to TS, illustrating the practical value of this design freedom. This flexibility opens the possibility of improving finite-horizon exploration–exploitation behavior while retaining the structural conditions needed for asymptotic guarantees.

\begin{figure}[!t]
\centering

\begin{subfigure}[t]{0.49\linewidth}
    \centering
    \includegraphics[width=\linewidth]{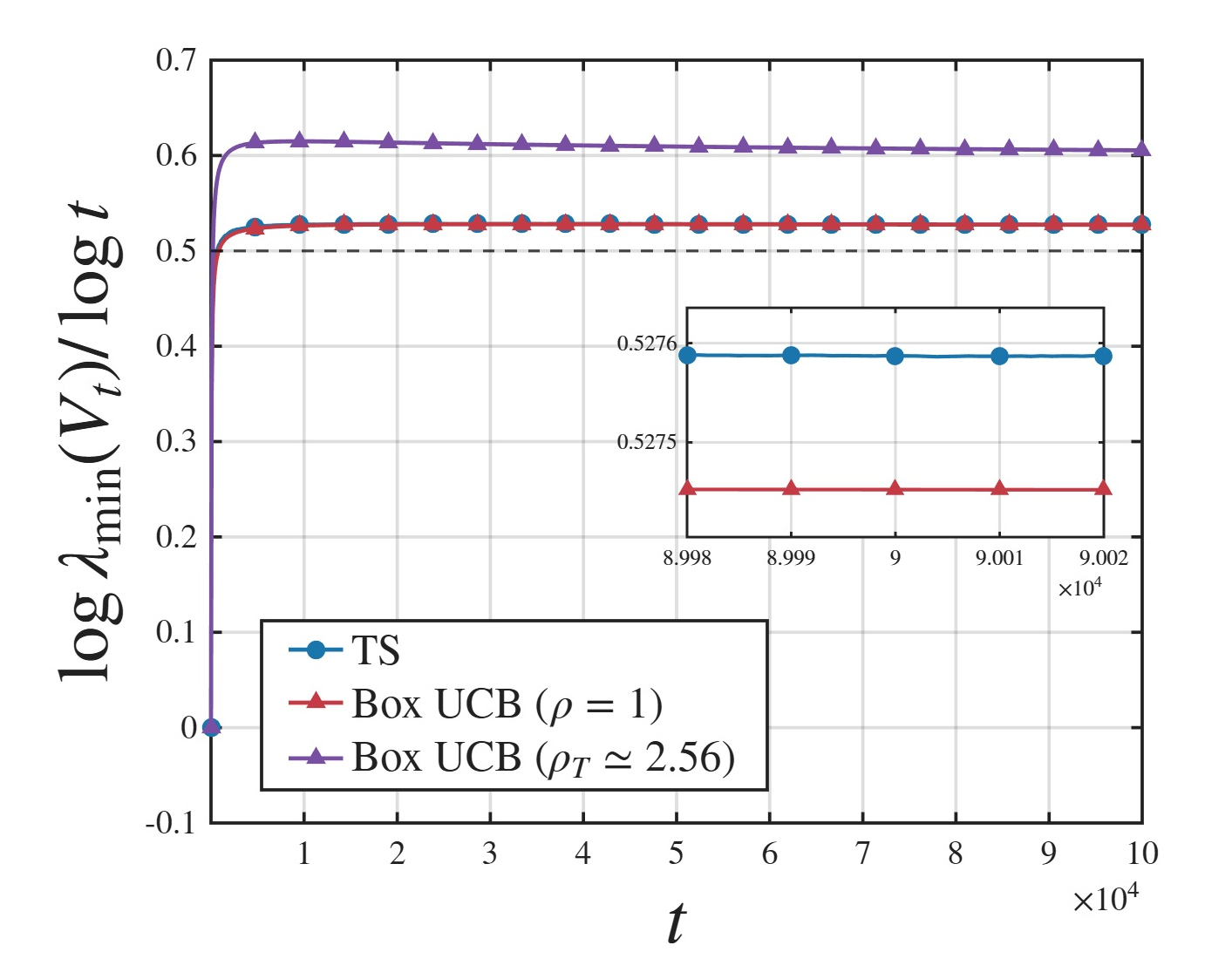}
    \caption{Minimum-eigenvalue growth:
    $\log\lambda_{\min,t}/\log t$.}
    \label{fig:l2_information}
\end{subfigure}
\hfill
\begin{subfigure}[t]{0.49\linewidth}
    \centering
    \includegraphics[width=\linewidth]{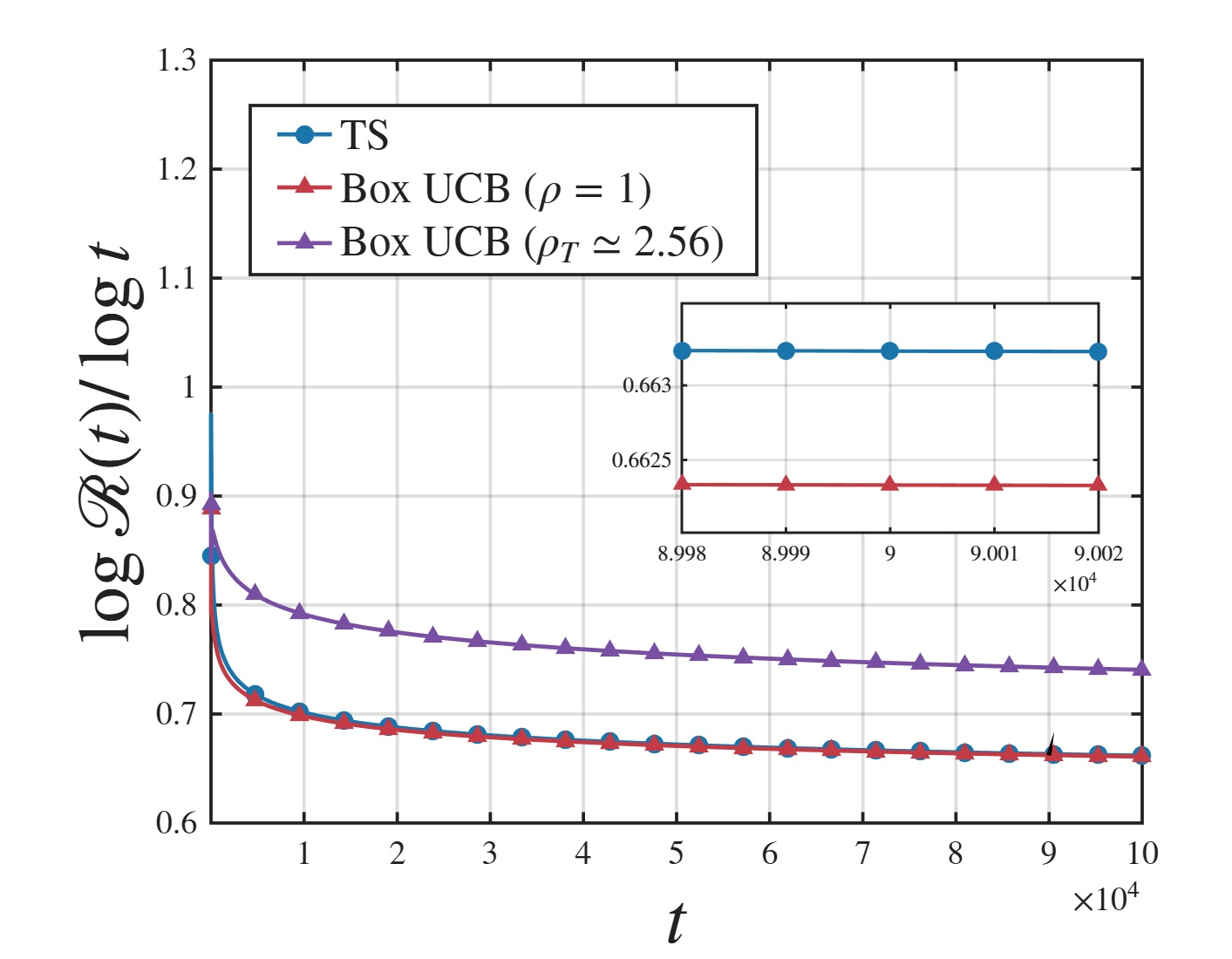}
    \caption{Regret growth:
    $\log\mathscr{R}_{\theta^\star}(t)/\log t$.}
    \label{fig:l2_regret}
\end{subfigure}

\caption{Experiment in $\mathbb R^{10}$ with
$T=10^5$, noise standard deviation $0.1$, and ten repetitions.
For each policy, the solid line is obtained by first averaging the
underlying quantity $X_t$ across repetitions to produce $\overline{X}_t$ and then plotting
$\log \overline{X}_t/\log t$, with
$X_t=\lambda_{\min,t}$ in panel~(a) and
$X_t=\mathscr{R}_{\theta^\star}(t)$ in panel~(b). Blue circles correspond to TS with $c_{\rm TS}=1$;
red triangles correspond to Box UCB with $\rho=1$;
and purple triangles correspond to Box UCB with the
horizon-dependent radius $\rho_T\approx2.56$.
The dashed horizontal line in panel~(a) marks the square-root benchmark
$1/2$.
Because the TS and fixed-radius Box UCB curves are nearly overlapping, each panel includes a magnified inset
centered at $t=0.9T=9\times10^4$, showing only the blue and red curves
for easier comparison.}
\label{fig:l2_numerics}
\end{figure}

\section{Conclusions}\label{sec:conclusion}

We have developed a unified perspective on inference quality and regret in linear bandits that identifies sufficient inference growth and action concentration as the key requirements for order-optimal regret, up to logarithmic factors. Central to this perspective is a self-correction mechanism that yields $\Omega(\sqrt{t})$ growth of $\lambda_{\min,t}$ under the stated assumptions. Together with controlled deviation from the estimated greedy action, this information growth bounds regret accumulation. The resulting framework accommodates both TS and UCB variants without requiring optimism as a premise of the analysis.

Our analysis focuses on the scaling of regret and inference guarantees
with elapsed time and the time horizon, and does not optimize the scalings on the problem
dimension. An important direction for future work is to determine whether the proposed technique
can recover order-optimal dimension dependence while preserving the time-dependent regret and
inference guarantees established here.

Furthermore, our results provide sufficient conditions for near-optimal regret, while leaving open the extent to which these conditions are also necessary. The connection between inference quality and regret established in \cite{trail}, including necessity results in expectation, motivates the study of analogous high-probability requirements. In particular, whether $\lambda_{\min,t}$ must grow at $\Omega(\sqrt{t})$ with high probability remains an open question.

Another direction is to extend the analysis beyond the Euclidean unit ball to more general action sets, including sublevel sets of strictly convex functions as considered in \cite{trail}. The geometric conditions studied in \cite{abeille2025when} may also help identify settings in which a more general optimism-agnostic regret analysis can be established.

Finally, we aim to investigate whether this optimism-agnostic approach extends to other adaptive decision and control problems, including linear quadratic regulators and Markov decision processes. The central challenge is to identify suitable measures of inference quality and the mechanisms that improve them while controlling the cost of exploration. Such extensions could provide a common basis for analyzing and designing exploration strategies across a broader class of sequential decision problems.

\bibliographystyle{IEEEtran}
\bibliography{adaptive}

\appendix

\end{document}